\documentclass{article}

\usepackage{amsmath,amssymb,amsthm,commath,dsfont,nicefrac}
\usepackage{mathtools,thmtools}
\usepackage{graphicx,wrapfig}
\usepackage{hyperref,cleveref}
\usepackage[round]{natbib}
\usepackage[margin=1in]{geometry}
\usepackage{xcolor}
\usepackage{comment}
\usepackage{enumitem}
\usepackage{tikz}
\usetikzlibrary{patterns}
\usepackage{authblk}

\theoremstyle{plain}
\newtheorem{theorem}{Theorem}[section]

\newtheorem{lemma}[theorem]{Lemma}

\theoremstyle{definition}
\newtheorem{definition}[theorem]{Definition}
\newtheorem{assumption}[theorem]{Assumption}
\theoremstyle{remark}

\def\ddefloop#1{\ifx\ddefloop#1\else\ddef{#1}\expandafter\ddefloop\fi}
\def\ddef#1{\expandafter\def\csname bb#1\endcsname{\ensuremath{\mathbb{#1}}}}
\ddefloop ABCDEFGHIJKLMNOPQRSTUVWXYZ\ddefloop
\def\ddef#1{\expandafter\def\csname c#1\endcsname{\ensuremath{\mathcal{#1}}}}
\ddefloop ABCDEFGHIJKLMNOPQRSTUVWXYZ\ddefloop

\DeclareMathOperator{\argmin}{arg\,min}

\DeclareMathOperator\poly{poly}

\def\1{\mathds{1}}
\def\R{\mathbb{R}}
\def\pr{\mathrm{Pr}}
\def\nR{\nabla\cR}
\def\hR{\widehat{\mathcal{R}}}
\def\nhR{\nabla\widehat{\mathcal{R}}}

\def\baru{\bar{u}}

\def\barw{\bar{w}}

\def\opt{\mathrm{OPT}}
\def\sign{\mathrm{sign}}
\def\barr{\bar{r}}
\def\llog{\ell_{\mathrm{log}}}

\def\lbu{Q}
\def\lbl{q}
\def\hw{\hat{w}}
\def\cRlog{\mathcal{R}_{\log}}
\def\nRlog{\nabla\mathcal{R}_{\log}}
\def\hRlog{\widehat{\mathcal{R}}_{\log}}
\def\nhRlog{\nabla\widehat{\mathcal{R}}_{\log}}
\def\epsopt{\epsilon_{\ell}}

\newcommand{\ip}[2]{\left\langle #1, #2 \right\rangle}

\title{Agnostic Learnability of Halfspaces via Logistic Loss\footnote{Part of this work was done when Ziwei Ji and Kwangjun Ahn were interns at Google.}}

\author[1]{Ziwei Ji\thanks{ziweiji2@illinois.edu}}
\author[2]{Kwangjun Ahn\thanks{kjahn@mit.edu}}
\author[3]{Pranjal Awasthi\thanks{pranjalawasthi@google.com}}
\author[3]{Satyen Kale\thanks{satyenkale@google.com}}
\author[3,4]{Stefani Karp\thanks{stefanik@google.com}}
\affil[1]{University of Illinois Urbana-Champaign}
\affil[2]{Massachusetts Institute of Technology}
\affil[3]{Google Research}
\affil[4]{Carnegie Mellon University}

\begin{document}

\maketitle

\begin{abstract}
  We investigate approximation guarantees provided by logistic regression for the fundamental problem of agnostic learning of homogeneous halfspaces. Previously, for a certain broad class of ``well-behaved'' distributions on the examples, \citet{diakonikolas_adv_noise}
  proved an $\widetilde{\Omega}(\opt)$ lower bound, while
  \citet{frei_soft_margin} proved an
  $\widetilde{O}\del[1]{\sqrt{\opt}}$ upper bound, where $\opt$ denotes the
  best zero-one/misclassification risk of a homogeneous halfspace.
  In this paper, we close this gap by constructing a well-behaved distribution
  such that the global minimizer of the logistic risk over this
  distribution only achieves $\Omega\del[1]{\sqrt{\opt}}$ misclassification
  risk, matching the upper bound in \citep{frei_soft_margin}.
  On the other hand, we also show that if we impose a radial-Lipschitzness condition in addition to well-behaved-ness on the distribution,
  logistic regression on a ball of bounded radius
  reaches
  $\widetilde{O}(\opt)$ misclassification risk.
  Our techniques also show for any well-behaved distribution, regardless of radial Lipschitzness,
  we can overcome the $\Omega(\sqrt{\opt})$ lower bound for logistic loss simply at the cost of one additional convex optimization step involving the hinge loss and attain $\widetilde{O}(\opt)$ misclassification risk.
  This two-step convex optimization algorithm is simpler than previous methods obtaining this guarantee, all of which require solving $O\del{\log(1/\opt)}$ minimization problems.
\end{abstract}

\section{Introduction}

In this paper, we consider the fundamental problem of agnostically learning homogeneous halfspaces. Specifically, we assume there is an unknown distribution $P$ over $\R^d\times\{-1,+1\}$ to which we have access in the form of independent and identically distributed samples drawn from $P$. Our goal is to compete with a homogeneous linear classifier $\baru$ (i.e. one that predicts the label $\sign(\langle \baru, x\rangle)$ for input $x$) that achieves the optimal zero-one risk of $\opt > 0$ over $P$.
Alternatively, we can think that the labels of the examples are first generated by $\baru$, and
then an $\opt$ fraction of the labels are adversarially corrupted.

There have been many algorithmic and hardness results on this topic, see
\Cref{sec:rw} for a discussion. A very natural heuristic for solving the problem is to use {\em logistic regression}. However, the analysis of
logistic regression for this problem is
still largely incomplete, even though it is one of the most fundamental
algorithms in machine learning.
One reason for this is that it can return {\em extremely poor} solutions in the worst case:
\citet{ben_hinge} showed that the minimizer of the
logistic risk may attain a zero-one risk as bad as $1-\opt$
on an adversarially-constructed distribution.

As a result, much attention has been devoted to certain ``well-behaved''
distributions, for which much better results can be obtained.
However, even when the marginal distribution on the feature space, $P_x$, is assumed to be isotropic log-concave, in recent work \citet{diakonikolas_adv_noise} proved an $\widetilde{\Omega}\del[1]{\opt}$ lower bound on the zero-one risk for any convex surrogate including logistic regression.
On the positive side, in another recent work, \citet{frei_soft_margin} proved that vanilla gradient descent on the logistic risk can attain a zero-one risk of
$\widetilde{O}\del[1]{\sqrt{\opt}}$, as long as $P_x$ satisfies some specific
well-behaved-ness conditions.
(For the details of such conditions, see \Cref{sec:rw,sec:log_ub}.)

The above results still leave a big gap between the upper and the lower bounds, raising the question of identifying the fundamental limits of logistic regression for this problem.
In this work we study this question and develop the following set of results.

\paragraph{A matching $\sqrt{\opt}$ lower bound.}

In \Cref{sec:log_lb}, we construct a distribution $\lbu$ over $\R^2 \times \{-1, 1\}$, and prove a lower bound for logistic regression
that matches the upper bound in \citep{frei_soft_margin}, thereby closing the gap in recent works \citep{diakonikolas_adv_noise,frei_soft_margin}.
Specifically, the marginal distribution $\lbu_x$ is isotropic and bounded,
and satisfies all the well-behaved-ness conditions from the aforementioned papers,
but the global minimizer of the logistic risk on $\lbu$
only attains $\Omega\del[1]{\sqrt{\opt}}$ zero-one risk on $Q$.

\paragraph{An $\widetilde{O}(\opt)$ upper bound for radially Lipschitz densities.}
The lower bound mentioned above shows that one needs to make additional assumptions to prove better bounds.
In \Cref{sec:log_ub}, we show that by making a radial Lipschitzness assumption in addition to well-behaved-ness, it is indeed possible to achieve the near-optimal
$\widetilde{O}(\opt)$ zero-one risk via logistic regression.
In particular, our upper bound result holds if the projection of $P_x$ onto any
two-dimensional subspace has Lipschitz continuous densities. Moreover, our upper bound analysis is versatile:
it can recover the
$\widetilde{O}\del[1]{\sqrt{\opt}}$ guarantee for general well-behaved
distributions shown by \citet{frei_soft_margin},
and it also works for the hinge loss, which motivates a simple and efficient two-phase algorithm, as described next.

\paragraph{An $\widetilde{O}(\opt)$ upper bound for general well-behaved
distributions with a two-phase algorithm.}

Motivated by our analysis, in \Cref{sec:hinge}, we describe a simple two-phase algorithm that achieves $\widetilde{O}\del[1]{{\opt}}$ error for general well-behaved distributions, without assuming radial Lipschitzness. Thus, we show that the cost of avoiding the radial Lipschitzness condition is simply an additional convex loss minimization. Our two-phase algorithm involves
logistic regression followed by stochastic gradient descent with the hinge loss
(i.e., the perceptron algorithm) with a restricted domain and warm start.
For general well-behaved distributions, the first phase can only achieve an
$\widetilde{O}\del[1]{\sqrt{\opt}}$ guarantee, however we show that the second
phase can boost the upper bound to $\widetilde{O}(\opt)$.

Previously, for any given $\epsilon > 0$,  \citet{diakonikolas_adv_noise} designed a nonconvex optimization
algorithm that can achieve $O(\opt+\epsilon)$ risk using
$\widetilde{O}(d/\epsilon^4)$ samples. Their algorithm requires guessing $\opt$ within a constant multiplicative factor via binary search and running a nonconvex SGD using each guess as an input.
Similarly, prior algorithms achieving $O(\opt+\epsilon)$ error involve solving multiple rounds of convex loss minimization \citep{awasthi_localization, daniely2015ptas}. By contrast, our two-phase algorithm is a simple logistic regression followed by
a perceptron algorithm, and the output is guaranteed to have $O\del{\opt\cdot\ln(1/\opt)+\epsilon}$ zero-one risk
using only $\widetilde{O}(d/\epsilon^2)$ samples.

\subsection{Related work}\label{sec:rw}
The problem of agnostic learning of halfspaces has a long and rich history \cite{kearns1994toward}. Here we survey the results most relevant to our work. It is well known that in the distribution independent setting, even {\em weak} agnostic learning is computationally hard \citep{feldman2006new, guruswami2009hardness, daniely2016complexity}. As a result most algorithmic results have been obtained under assumptions on the marginal distribution $P_x$ over the examples.

The work of \citet{kalai2008agnostically} designed algorithms that achieve $\opt + \epsilon$ error for any $\epsilon > 0$ in time $d^{\text{poly}(\frac 1 \epsilon)}$ for isotropic log-concave densities and for the uniform distribution over the hypercube. There is also recent evidence that removing the exponential dependence on $1/\epsilon$, even for Gaussian marginals is computationally hard \citep{klivans2014embedding, diakonikolas2020near,goel_statistical}.

As a result, another line of work aims to design algorithms with polynomial running time and sample complexity~(in $d$ and $\frac 1 \epsilon$) and achieve an error of $g(\opt) + \epsilon$, for $g$ being a simple function. Along these lines \citet{klivans2009learning} designed a polynomial-time algorithm that attains
$\widetilde{O}(\opt^{1/3})+\epsilon$ zero-one risk for isotropic log-concave
distributions.
\citet{awasthi_localization} improved the upper bound to $O(\opt)+\epsilon$,
using a localization-based algorithm.
\citet{balcan_s_concave} further extended the algorithm to more general $s$-concave
distributions. The work of \citet{daniely2015ptas} further provided a {\em PTAS} guarantee: an error of $(1+\eta)\opt + \epsilon$ for any desired constant $\eta > 0$ via an improper learner.

In a recent work \citet{diakonikolas_adv_noise} studied the problem for distributions satisfying certain ``well-behaved-ness'' conditions which include isotropy and certain regularity conditions on the projection of $P_x$ on any 2-dimensional subspace (see Assumption~\ref{cond:well} for a subset of these conditions). This class of distributions include any isotropic log-concave distribution such as the standard Gaussian.
In addition to their nonconvex optimization method discussed above,
for any convex, nonincreasing, and nonconstant loss function,
they also showed an $\Omega\del{\opt\ln(1/\opt)}$ lower bound for
log-concave marginals and an $\Omega\del[1]{\opt^{1-1/s}}$ lower bound for
$s$-heavy-tailed marginals.

In another recent work \citet{frei_soft_margin} assumed $P_x$ satisfies a ``soft-margin'' condition:
for anti-concentrated marginals such as isotropic log-concave marginals, this assumes
$\pr\del[1]{\,\envert{\langle\baru,x\rangle}\le\gamma}=O(\gamma)$ for any $\gamma>0$.
For sub-exponential distributions with soft-margins, they proved an $\widetilde{O}\del[1]{\sqrt{\opt}}$ upper
bound for gradient descent on the logistic loss, which can be improved to $O\del[1]{\sqrt{\opt}}$ for
bounded distributions.
Note that these upper bounds and the lower bounds in
\citep{diakonikolas_adv_noise} do not match: if $P_x$ is
sub-exponential, then \citet{diakonikolas_adv_noise} only gave an
$\widetilde{\Omega}(\opt)$ lower bound, while if $P_x$ is $s$-heavy-tailed, then
the upper bound in \citep{frei_soft_margin} becomes worse.

Finally, some prior works on agnostic learning of halfspaces have considered various extensions of the problem such as active agnostic learning \citep{awasthi_localization,yan2017revisiting}, agnostic learning of sparse halfspaces with sample complexity scaling logarithmically in the ambient dimensionality \citep{shen2021attribute}, and agnostic learning under weaker noise models such as the random classification noise \citep{blum1998polynomial, dunagan2008simple}, Massart's noise model \citep{awasthi2015efficient, awasthi2016learning, zhang2020efficient, diakonikolas2019distribution, diakonikolas2020learning, diakonikolas2021threshold, chen2020classification} and the Tsybakov noise model \citep{diakonikolas2020learningtsybakov, zhang2021improved}. We do not consider these extensions in our work.

\subsection{Notation}

Let $\|\cdot\|$ denote the $\ell_2$ (Euclidean) norm.
Given $r>0$, let $\cB(r):=\cbr{x\middle|\|x\|\le r}$ denote the Euclidean
ball with radius $r$.
Given two nonzero vectors $u$ and $v$, let $\varphi(u,v)\in[0,\pi]$ denote the
angle between them.

Given a data distribution $P$ over $\R^d\times\{-1,+1\}$, let $P_x$ denote the
marginal distribution of $P$ on the feature space $\R^d$.
We will frequently need the projection of the input features onto a two-dimensional
subspace $V$; in such cases, it will be convenient to use
polar coordinates $(r,\theta)$ for the associated calculations, such as parameterizing the density with respect to the Lebesgue measure as $p_V(r, \theta)$.

Given a nonincreasing loss function $\ell: \bbR \rightarrow \bbR$, we consider the population risk
\begin{align*}
  \cR_\ell(w):=\bbE_{(x,y)\sim P}\sbr[1]{\ell\del{y \langle w,x\rangle}},
\end{align*}
and the corresponding empirical risk
\begin{align*}
  \hR_\ell(w):=\frac{1}{n}\sum_{i=1}^{n}\ell\del{y_i \langle w,x_i\rangle},
\end{align*}
defined over $n$ i.i.d. samples drawn from $P$.
We will focus on the logistic loss $\llog(z):=\ln(1+e^{-z})$, and the
hinge loss $\ell_h(z):=\max\{-z,0\}$.
Let $\cRlog:=\cR_{\llog}$ for simplicity, and also define $\hRlog$, $\cR_h$ and
$\hR_h$ similarly.
Let $\cR_{0-1}(w):=\pr_{(x,y)\sim P}\del[1]{y\ne\sign\del{\langle w,x\rangle}}$
denote the population zero-one risk.

\section{An $\Omega\del[1]{\sqrt{\opt}}$ lower bound for logistic loss}\label{sec:log_lb}

In this section, we construct a distribution $\lbu$ over $\R^2\times\{-1,+1\}$
which satisfies standard regularity conditions in \citep{diakonikolas_adv_noise,frei_adv_noise}, but the global minimizer $w^*$ of the
population logistic risk $\cRlog$ on $\lbu$ only achieves a zero-one risk of $\Omega\del[1]{\sqrt{\opt}}$.
Our focus on the global logistic optimizer is motivated by the lower bounds from \citep{diakonikolas_adv_noise}; in particular, this means that the large
classification error is not caused by the sampling error.

\begin{wrapfigure}{R}{0.4\textwidth}
\vspace{1em}
\centering
\begin{tikzpicture}[scale=2]

\draw[->,thick] (-1.5,0) -- (1.5,0);
\draw[->,thick] (0,-1.2) -- (0,1.2);

\draw [thick,pattern=horizontal lines,pattern color=red] (0,-1) arc (-90:90:1);
\draw [thick,pattern=horizontal lines,pattern color=blue] (0,1) arc (90:270:1);

\draw [thick,pattern=vertical lines,pattern color=red] (0,0) rectangle (0.25,1);
\draw [thick,pattern=vertical lines,pattern color=blue] (-0.25,-1) rectangle (0,0);
\draw (0.25,1) node[anchor=south west]{$\lbu_2$};
\draw (-0.25,-1) node[anchor=north east]{$\lbu_2$};

\draw [thick,pattern=north west lines,pattern color=blue] (0.619,-0.795) rectangle (0.795,-0.619);
\draw [thick,pattern=north west lines,pattern color=red] (-0.795,0.619) rectangle (-0.619,0.795);
\draw (0.795,-0.795) node[anchor=north west]{$\lbu_1$};
\draw (-0.795,0.795) node[anchor=south east]{$\lbu_1$};

\draw [thick,pattern=north east lines,pattern color=red] (0.860,-0.139) rectangle (1.139,0.139);
\draw [thick,pattern=north east lines,pattern color=blue] (-1.139,-0.139) rectangle (-0.860,0.139);
\draw (1.139,0.139) node[anchor=west]{$\lbu_3$};
\draw (-1.139,-0.139) node[anchor=east]{$\lbu_3$};
\end{tikzpicture}
\caption{An illustration of $\lbu$ when $\opt=1/16$.
Red areas denote positive examples, while blue areas denote negative examples.
The parts $\lbu_1$, $\lbu_2$ and $\lbu_3$ are marked in the figure, while
$\lbu_4$ is supported on the unit circle and marked by horizontal lines.}
\vspace{-5em}
\end{wrapfigure}
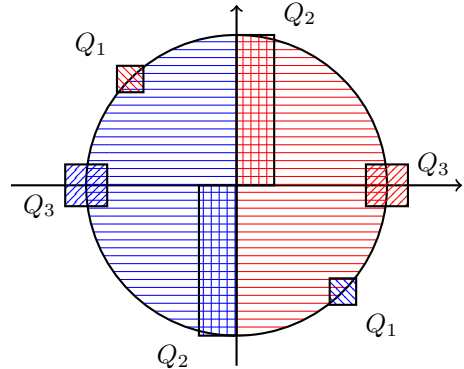

The distribution $\lbu$ has four parts $\lbu_1$, $\lbu_2$, $\lbu_3$, and
$\lbu_4$, as described below.
It can be verified that if $\opt\le1/16$, the construction is valid.
\begin{enumerate}
  \item The feature distribution of $\lbu_1$ consists of two squares: one has
  edge length $\sqrt{\frac{\opt}{2}}$, center
  $\del{\frac{\sqrt{2}}{2},-\frac{\sqrt{2}}{2}}$ and density $1$, with label
  $-1$; the other has edge length $\sqrt{\frac{\opt}{2}}$, center
  $\del{-\frac{\sqrt{2}}{2},\frac{\sqrt{2}}{2}}$, density $1$, with label
  $+1$.

  \item The feature distribution of $\lbu_2$ is supported on
  \begin{align*}
    \del[2]{\sbr{0,\sqrt{\opt}}\times[0,1]}\cup\del[2]{\sbr{-\sqrt{\opt},0}\times[-1,0]}
  \end{align*}
  with density $1$, and the label is given by $\sign(x_1)$.

  \item Let $\lbl_3:=\frac{2}{3}\sqrt{\opt}(1-\opt)$, then $\lbu_3$ consists of
  two squares: one has edge length $\sqrt{\frac{\lbl_3}{2}}$, center $(1,0)$,
  density $1$ and label $+1$, and the other has edge length
  $\sqrt{\frac{\lbl_3}{2}}$, center $(-1,0)$, density $1$ and label $-1$.

  \item The feature distribution of $\lbu_4$ is the uniform distribution over
  the unit ball $\cB(1):=\cbr{x\middle|\|x\|\le1}$ with density
  $\lbl_4:=\frac{1-\opt-2\sqrt{\opt}-\lbl_3}{\pi}$, and the label is given by
  $\sign(x_1)$.
\end{enumerate}

Note that the correct label is given by $\sign(x_1)$ on $\lbu_2$, $\lbu_3$ and
$\lbu_4$; therefore $\baru:=(1,0)$ is our ground-truth solution that is only
wrong on the noisy part $\lbu_1$.

Here is our lower bound result.
\begin{theorem}\label{fact:log_lb}
  Suppose $\opt\le1/100$, and let $\lbu_x$ denote the marginal distribution of
  $\lbu$ on the feature space.
  It holds that $\bbE_{x\sim\lbu_x}[x]=0$, and $\bbE_{x\sim\lbu_x}[x_1x_2]=0$,
  and $\bbE_{x\sim\lbu_x}[x_1^2-x_2^2]=0$.
  Moreover, the population logistic risk $\cRlog$ has a global minimizer
  $w^*$, and
  \begin{align*}
    \cR_{0-1}(w^*)=\pr\del{y\ne\sign\del{\langle w^*,x\rangle}}\ge \frac{\sqrt{\opt}}{60\pi}.
  \end{align*}
\end{theorem}

Note that we can further normalize $\lbu_x$ to unit variance and make it
isotropic.
Then it is easy to check that $\lbu_x$ satisfies the ``well-behaved-ness'' conditions in
\citep{diakonikolas_adv_noise}, and the ``soft-margin'' and
``sub-exponential'' conditions in \citep{frei_soft_margin}.
In particular, our lower bound matches the upper bound in
\citep{frei_soft_margin}.

\subsection{Proof of \Cref{fact:log_lb}}

Here is a proof sketch of \Cref{fact:log_lb}; the full proof is given in \Cref{app_sec:log_lb}.

First, basic calculation shows that $\lbu$ is isotropic up to a constant multiplicative
factor, and that $\lbu$ has bounded density and support.
Specifically, $\lbu_1$, $\lbu_2$ and $\lbu_4$ are constructed to make the risk
lower bound proof work, while $\lbu_3$ is included to make $\lbu$ isotropic.
It turns out that $\lbu_3$ does not change the risk lower bound proof too much,
since it is highly aligned with the ground-truth solution $\baru:=(1,0)$.

Next we consider the risk lower bound.
We only need to show that $\varphi(\baru,w^*)$, the angle between $\baru$ and
$w^*$, is $\Omega\del[1]{\sqrt{\opt}}$, since it then follows that $w^*$ is
wrong on an $\Omega\del[1]{\sqrt{\opt}}$ fraction of $\lbu_4$, which is enough
since $\lbu_4$ accounts for more than half of the distribution $\lbu$.

Note that the minimizer of the logistic risk on $\lbu_4$ by itself is infinitely far in the direction of $\baru$.
However, this will incur a large risk on $\lbu_1$.
By balancing these two parts, we can show that by moving along the direction of
$\baru$ by a distance of $\Theta\del{\frac{1}{\sqrt{\opt}}}$, we can achieve a
logistic risk of $O\del[1]{\sqrt{\opt}}$.
\begin{lemma}\label{fact:log_lb_ref_sol}
  Suppose $\opt\le1/100$, let $\barw:=(\barr,0)$ where
  $\barr=\frac{3}{\sqrt{\opt}}$, then $\cRlog(\barw)\le5\sqrt{\opt}$.
\end{lemma}

Next we consider the global minimizer $w^*$ of $\cRlog$, which exists since
$\cRlog$ has bounded sub-level sets.
Let $(r^*,\theta^*)$ denote the polar coordinates of $w^*$.
We will assume
$\theta^*\in\sbr{-\frac{\sqrt{\opt}}{30},\frac{\sqrt{\opt}}{30}}$, and derive a
contradiction.
In our construction, $\lbu_3$ and $\lbu_4$ are symmetric with respect to the
horizontal axis, and they will induce the ground-truth solution.
However, $\lbu_1$ and $\lbu_2$ are skew, and they will pull $w^*$ above, meaning
we actually have $\theta^*\in\sbr{0,\frac{\sqrt{\opt}}{30}}$.
The first observation is an upper bound on $r^*$: if $r^*$ is too large, then
the risk of $w^*$ over $\lbu_1$ will already be larger than $\cRlog(\barw)$
for $\barw$ constructed in \Cref{fact:log_lb_ref_sol}, a contradiction.
\begin{lemma}\label{fact:w*_norm_ub}
  Suppose $\opt\le1/100$ and $\theta^*\in\sbr{0,\frac{\sqrt{\opt}}{30}}$, then
  $r^*\le \frac{10}{\sqrt{\opt}}$.
\end{lemma}

However, our next lemma shows that under the above conditions, the gradient of
$\cRlog$ at $w^*$ does not vanish, which contradicts the definition of $w^*$.
\begin{lemma}\label{fact:nonzero_grad}
  Suppose $\opt\le1/100$, then for any $w=(r,\theta)$ with
  $0\le r\le \frac{10}{\sqrt{\opt}}$ and $0\le\theta\le \frac{\sqrt{\opt}}{30}$,
  it holds that $\nRlog(w)\ne0$.
\end{lemma}
To prove \Cref{fact:nonzero_grad}, let us consider an arbitrary $w=(r,\theta)$
under the conditions of \Cref{fact:nonzero_grad}.
For simplicity, let us first look at the case $\theta=0$.
In this case, $y \langle w,x\rangle\le10$ on $\lbu_2$, and it follows that
$\lbu_2$ induces a component of length $\frac{C_1}{\sqrt{\opt}}$ in the
gradient $\nRlog(w)$ along the direction of $-e_2=(0,-1)$, where $C_1$ is a
universal constant.
Moreover, $\lbu_1$ also induces a component in the gradient along
$-e_2$, while $\lbu_3$ and $\lbu_4$ induce a zero component along $e_2$.
As a result, $\ip{\nRlog(w)}{e_2}<0$.
Now if $0\le\theta\le C_2\sqrt{\opt}$ for some small enough constant $C_2$
($1/30$ in our case), we can show that $\lbu_3$ and $\lbu_4$ cannot cancel the
effect of $\lbu_2$, and it still holds that $\ip{\nRlog(w)}{e_2}<0$.

\section{An $\widetilde{O}(\opt)$ upper bound for logistic loss with radial Lipschitzness}\label{sec:log_ub}

The $\Omega(\sqrt{\opt})$ lower bound construction in Section~\ref{sec:log_lb} shows that further assumptions on the distribution are necessary in order to improve the upper bound on the zero-one risk of the logistic regression solution. In particular, we note that the distribution $\lbu$ constructed in Section~\ref{sec:log_lb} has a {\em discontinuous} density. In this section, we show that if we simply add a very mild Lipschitz continuity condition on the density, then we can achieve $\widetilde{O}(\opt)$ zero-one risk using logistic regression.

First, we formally provide the standard assumptions from prior work.
Because of the lower bound for $s$-heavy-tailed distributions from
\citep{diakonikolas_adv_noise}, to get an $\widetilde{O}(\opt)$ zero-one risk,
we need to assume $P_x$ has a light tail.
Following \citep{frei_soft_margin}, we will either consider a bounded
distribution, or assume $P_x$ is sub-exponential as defined below (cf.
\citep[Proposition 2.7.1 and Section 3.4.4]{roman_hdp}).
\begin{definition}\label{def:sub_exp}
  We say $P_x$ is $(\alpha_1,\alpha_2)$ sub-exponential for constants
  $\alpha_1,\alpha_2>0$, if for any unit vector $v$ and any $t>0$,
  \begin{align*}
    \pr_{x\sim P_x}\del[2]{\envert{\langle v,x\rangle}\ge t}\le\alpha_1\exp\del{-t/\alpha_2}.
  \end{align*}
\end{definition}

We also need the next assumption, which is part of the ``well-behaved-ness''
conditions from \citep{diakonikolas_adv_noise}.
\begin{assumption}\label{cond:well}
  There exist constants $U,R>0$ and a function $\sigma:\R_+\to\R_+$,
  such that if we project $P_x$ onto an arbitrary two-dimensional subspace $V$,
  the corresponding density $p_V$ satisfies
  $p_V(r,\theta)\ge1/U$ for all $r\le R$, and
  $p_V(r,\theta)\le\sigma(r)$ for all $r\ge0$,
  and $\int_0^\infty\sigma(r)\dif r\le U$, and
  $\int_0^\infty r\sigma(r)\dif r\le U$.
\end{assumption}

Note that for a broad class of distributions including isotropic
log-concave distributions, the sub-exponential condition and
\Cref{cond:well} hold with $\alpha_1,\alpha_2,U,R$ all being
universal constants.

Finally, as discussed earlier, the previous conditions are also satisfied by
 $\lbu$ from \Cref{sec:log_lb}, and thus to get the improved $\widetilde{O}(\opt)$ risk bound, we need
the following radial Lipschitz continuity assumption.
\begin{assumption}\label{cond:rad}
  There exists a measurable function $\kappa:\R_+\to\R_+$ such that for any
  two-dimensional subspace $V$,
  \begin{align*}
    \envert{p_V(r,\theta)-p_V(r,\theta')}\le\kappa(r)|\theta-\theta'|.
  \end{align*}
\end{assumption}
We will see \Cref{cond:rad} is crucial for the upper bound analysis in
\Cref{fact:rotate_diff_main}.
For some concrete examples, note that if $p_V$ is radially symmetric, then
$\kappa$ is $0$, while if  $p_V$ is $\lambda$-Lipschitz continuous in the usual sense (under $\ell_2$), then we can
let $\kappa(r):=\lambda r$.

Now we can state our main results. In the following, we denote the unit linear classifier with the optimal zero-one risk by $\baru$, with $\cR_{0-1}(\baru)=\opt\in(0,1/e)$.
Our first result shows that, with \Cref{cond:rad}, minimizing the logistic risk
yields a solution with $\widetilde{O}(\opt)$ zero-one risk.

\begin{theorem}\label{thm:exact-opt}
Under Assumptions \ref{cond:well} and \ref{cond:rad}, let $w^*$ denote the global minimizer of $\cRlog$.
\begin{enumerate}
    \item If $\|x\|\le B$ almost surely, then
    \[\cR_{0-1}(w^*)=O\del{(1+C_\kappa)\opt},\]
    where $C_\kappa:=\int_0^B\kappa(r)\dif r$.

    \item If $P_x$ is $(\alpha_1,\alpha_2)$-sub-exponential, then
    \[\cR_{0-1}(w^*)=O\del[1]{(1+C_\kappa)\opt\cdot\ln(1/\opt)},\]
    where $C_\kappa:=\int_0^{3\alpha_2\ln(1/\opt)}\kappa(r)\dif r$.
  \end{enumerate}
\end{theorem}
Note that for bounded distributions, $\cR_{0-1}(w^*)=O(\opt)$ as long as
$C_\kappa$ is a constant, which is true if $p_V$ is Lipschitz continuous.
Similarly, for sub-exponential distributions, $C_\kappa$ is still
$O\del[1]{\ln(1/\opt)^2}$ for Lipschitz continuous densities, in
which case $\cR_{0-1}(w^*)=\widetilde{O}(\opt)$.

Next we give an algorithmic result.
Given a target error $\epsilon\in(0,1)$, we consider projected gradient descent
on the empirical risk with a norm bound of $1/\sqrt{\epsilon}$: let $w_0:=0$, and
\begin{align}\label{eq:pgd}
  w_{t+1}:=\Pi_{\cB(1/\sqrt{\epsilon})}\sbr{w_t-\eta\nhRlog(w_t)}.
\end{align}
Our next result shows that projected gradient descent can give an
$\widetilde{O}(\opt)$ risk.
Note that for the two cases discussed below (bounded or sub-exponential), we use
the corresponding $C_\kappa$ defined in \Cref{thm:exact-opt}.
\begin{theorem}\label{fact:log_ub_opt}
  Suppose Assumptions \ref{cond:well} and \ref{cond:rad} hold.
  \begin{enumerate}
    \item If $\|x\|\le B$ almost surely,
    then with $\eta=4/B^2$, and $O\del{\frac{\ln(1/\delta)}{\epsilon^4}}$
    samples and $O\del{\frac{1}{\epsilon^{5/2}}}$ iterations, with probability
    $1-\delta$, projected gradient descent outputs $w_t$ satisfying
    \begin{align*}
      \cR_{0-1}(w_t)=O\del{(1+C_\kappa)(\opt+\epsilon)}.
    \end{align*}

    \item On the other hand, if $P_x$ is $(\alpha_1,\alpha_2)$-sub-exponential,
    then with $\eta=\widetilde{\Theta}(1/d)$, using
    $\widetilde{O}\del{\frac{d\ln(1/\delta)^3}{\epsilon^4}}$ samples and
    $\widetilde{O}\del{\frac{d\ln(1/\delta)^2}{\epsilon^{5/2}}}$ iterations,
    with probability $1-\delta$, projected gradient descent outputs $w_t$ with
    \begin{align*}
       \cR_{0-1}(w_t)=\del{(1+C_\kappa)\del{\opt\cdot\ln(1/\opt)+\epsilon}}.
    \end{align*}
  \end{enumerate}
\end{theorem}

Theorems~\ref{thm:exact-opt} and \ref{fact:log_ub_opt} rely on the following key lemma which provides a zero-one risk bound on near optimal solutions to the logistic regression problem:
\begin{lemma}\label{fact:log_ub}
  Under Assumptions \ref{cond:well} and \ref{cond:rad}, suppose $\hw$ satisfies
  $\cRlog(\hw)\le\cRlog(\|\hw\|\baru)+\epsopt$ for some $\epsopt\in[0,1)$.
  \begin{enumerate}
    \item If $\|x\|\le B$ almost surely, then
    \begin{align*}
      \cR_{0-1}(\hw)= O\del[3]{\max\cbr{\opt,\sqrt{\tfrac{\epsopt}{\|\hw\|}},\tfrac{C_\kappa}{\|\hw\|^2}}}.
    \end{align*}

    \item If $P_x$ is $(\alpha_1,\alpha_2)$-sub-exponential and $\|\hw\|=\Omega(1)$,
    then
    \begin{align*}
      \cR_{0-1}(\hw)=O\del[3]{\max\cbr{\opt\cdot\ln(1/\opt),\sqrt{\tfrac{\epsopt}{\|\hw\|}},\tfrac{C_\kappa}{\|\hw\|^2}}}.
    \end{align*}
  \end{enumerate}
\end{lemma}

Next we give proof outlines of our results; the full proofs are given in
\Cref{app_sec:log_ub}.
For simplicity, here we focus only on the bounded case, while the sub-exponential
case will be handled in \Cref{app_sec:log_ub}.
Although the proofs of the two cases share some similarity, we want to emphasize
that the sub-exponential case does not follow by simply truncating the
distribution to a certain radius and reduce it to the bounded case.
The reason is that the truncation radius will be as large as $\sqrt{d}$, while for
the bounded case in \Cref{fact:log_ub}, $B$ is considered a constant independent of $d$ and hidden
in the $O$ notation; therefore this truncation argument will introduce a
$\poly(d)$ dependency in the final bound.
By contrast, our zero-one risk upper bounds only depend on
$\alpha_1$, $\alpha_2$, $U$ and $R$, but do not depend on $d$.

\subsection{Proof of \Cref{thm:exact-opt,fact:log_ub_opt}}\label{sec:log_ub_main}

Here we prove \Cref{thm:exact-opt,fact:log_ub_opt}; the details are given in
\Cref{app_sec:log_ub_main}.

We first prove \Cref{thm:exact-opt}.
Note that by \Cref{fact:log_ub}, it suffices to show that $\|w^*\|=\Omega\del{\frac{1}{\sqrt{\opt}}}$ (since $\epsopt=0$ in this case),
which is true due to the next result.
\begin{lemma}\label{fact:w*_norm_lb_main}
  Under \Cref{cond:well}, if $\|x\|\le B$ almost surely
  and $\opt<\frac{R^4}{200U^3B}$, then
  $\|w^*\|=\Omega\del[2]{\frac{1}{\sqrt{\opt}}}$.
\end{lemma}

Next we prove \Cref{fact:log_ub_opt}.
Recall that given the target (zero-one) error $\epsilon\in(0,1)$, we run
projected gradient descent on a Euclidean ball with radius $1/\sqrt{\epsilon}$
(cf. \cref{eq:pgd}).
Using a standard optimization and generalization analysis, we can prove the
following guarantee on $\cRlog(w_t)$.
\begin{lemma}\label{fact:pgd_wt_baru_main}
  Let the target optimization error $\epsopt\in(0,1)$ and the failure
  probability $\delta\in(0,1/e)$ be given.
  If $\|x\|\le B$ almost surely, then with $\eta=4/B^2$, using
  $O\del{\frac{(B+1)^2\ln(1/\delta)}{\epsilon\epsopt^2}}$ samples and
  $O\del{\frac{B^2}{\epsilon\epsopt}}$ iterations, with probability $1-\delta$,
  projected gradient descent outputs $w_t$ satisfying
  \begin{align}\label{eq:pgd_excess_main}
    \cRlog(w_t)\le\min_{0\le\rho\le1/\sqrt{\epsilon}}\cRlog(\rho\baru)+\epsopt.
  \end{align}
\end{lemma}

We also need the following  lower bounds on $\|w_t\|$.
\begin{lemma}\label{fact:wt_norm_lb_main}
  Under \Cref{cond:well}, suppose
  \begin{align*}
    \epsilon<\min\cbr{\frac{R^4}{36U^2},\frac{R^4}{72^2U^4}}\quad\textup{and}\quad\epsopt\le\sqrt{\epsilon},
  \end{align*}
  and that \cref{eq:pgd_excess_main} holds.
  If $\|x\|\le B$ almost surely and $\opt<\frac{R^4}{500U^3B}$,
  then $\|w_t\|=\Omega\del[2]{\min\cbr{\frac{1}{\sqrt{\epsilon}},\frac{1}{\sqrt{\opt}}}}$.
\end{lemma}

Now to prove \Cref{fact:log_ub_opt}, we simply need to combine
\Cref{fact:log_ub,fact:pgd_wt_baru_main,fact:wt_norm_lb_main} with
$\epsopt=\epsilon^{3/2}$.

\subsection{Proof of \Cref{fact:log_ub}}\label{sec:log_ub_01}

Here we give a proof sketch of \Cref{fact:log_ub}; the details are given in
\Cref{app_sec:log_ub_01}.
One remark is that some of the lemmas in the proof are also true for the hinge
loss, and this fact will be crucial in the later discussion regarding our two-phase algorithm (cf. \Cref{sec:hinge}).

Let $\barw:=\|\hw\|\baru$, and consider $\ell\in\{\llog,\ell_h\}$.
The first step is to express $\cR_\ell(\hw)-\cR_\ell(\barw)$ as the sum of three
terms, and then bound them separately.
The first term is given by
\begin{align}
  \cR_\ell(\hw)-\cR_\ell(\barw)-\bbE\sbr[2]{\ell\del{\sign\del{\langle\barw,x\rangle}\langle\hw,x\rangle}-\ell\del{\sign\del{\langle\barw,x\rangle}\langle\barw,x\rangle}}, \label{eq:log_diff_y_baru}
\end{align}
the second term is given by
\begin{align}\label{eq:log_diff_approx_equiv_1}
  \bbE\sbr[2]{\ell\del{\sign\del{\langle\barw,x\rangle}\langle\hw,x\rangle}-\ell\del{\sign\del{\langle\hw,x\rangle}\langle\hw,x\rangle}},
\end{align}
and the third term is given by
\begin{align}\label{eq:log_diff_approx_equiv_2}
  \bbE\sbr[2]{\ell\del{\sign\del{\langle\hw,x\rangle}\langle\hw,x\rangle}-\ell\del{\sign\del{\langle\barw,x\rangle}\langle \barw,x\rangle}},
\end{align}
where the expectations are taken over $P_x$.

We first bound term~\eqref{eq:log_diff_y_baru}, which is the approximation error
of replacing the true label $y$ with the label given by $\baru$.
Since $\ell(-z)-\ell(z)=z$ for the logistic loss and hinge loss, it follows that
\begin{align*}
  \textup{term~\eqref{eq:log_diff_y_baru}}=\bbE\sbr{\1_{y\ne\sign\del{\langle\barw,x\rangle}}\cdot y \langle\barw-\hw,x\rangle}.
\end{align*}
The approximation error can be bounded as below, using the tail bound on $P_x$
and the fact $\cR_{0-1}(\barw)=\opt$.
\begin{lemma}\label{fact:noisy_label_main}
 For $\ell\in\{\llog,\ell_h\}$, if $\|x\|\le B$ almost surely,
  \begin{align*}
    \envert{\textup{term~\eqref{eq:log_diff_y_baru}}}\le B\|\barw-\hw\|\cdot\opt.
  \end{align*}
\end{lemma}

Next we bound term~\eqref{eq:log_diff_approx_equiv_1}.
\begin{lemma}\label{fact:ground_truth_diff} Under \Cref{cond:well},  for $\ell\in\{\llog,\ell_h\}$,
  \begin{align*}
    \textup{term~\eqref{eq:log_diff_approx_equiv_1}}\ge \frac{4R^3}{3U\pi^2}\|\hw\|\varphi(\hw,\barw)^2.
  \end{align*}
\end{lemma}

Lastly, we consider term~\eqref{eq:log_diff_approx_equiv_2}.
Note that it is $0$ for the hinge loss $\ell_h$, because $\ell_h(z)=0$ when
$z\ge0$.
For the logistic loss, term~\eqref{eq:log_diff_approx_equiv_2} is also $0$ if
$P_x$ is radially symmetric; in general, we will bound it using \Cref{cond:rad}.
\begin{lemma}\label{fact:rotate_diff_main}
  For $\ell=\ell_h$, term~\eqref{eq:log_diff_approx_equiv_2} is $0$.
  For $\ell=\llog$, under \Cref{cond:rad}, if $\|x\|\le B$ almost surely,
  then{}
  \begin{align*}
    \envert{\textup{term~\eqref{eq:log_diff_approx_equiv_2}}}\le12C_\kappa\cdot\varphi(\hw,\barw)/\|\hw\|,
  \end{align*}
  where $C_\kappa:=\int_0^B\kappa(r)\dif r$.
\end{lemma}

Now we are ready to prove \Cref{fact:log_ub}.
For simplicity, here we let $\varphi$ denote $\varphi(\hw,\barw)$.
For bounded distributions,
\Cref{fact:noisy_label_main,fact:ground_truth_diff,fact:rotate_diff_main} imply
\begin{align*}
  C_1\|\hw\|\varphi^2\le\epsopt+B\|\barw-\hw\|\cdot\opt+C_2C_\kappa\cdot\varphi/\|\hw\|\le\epsopt+B\|\hw\|\varphi\cdot\opt+C_2C_\kappa\cdot\varphi/\|\hw\|,
\end{align*}
where $C_1=4R^3/(3U\pi^2)$ and $C_2=12$.
It follows that at least one of the following three cases is true:
\begin{enumerate}
  \item $C_1\|\hw\|\varphi^2\le3\epsopt$, which implies
 $\varphi=O\del[1]{\sqrt{\epsopt/\|\hw\|}}$;

 \item $C_1\|\hw\|\varphi^2\le3B\|\hw\|\varphi\cdot\opt$, and it follows that
 $\varphi=O(\opt)$;

 \item $C_1\|\hw\|\varphi^2\le3C_2C_\kappa\cdot\varphi/\|\hw\|$, and it follows
 that $\varphi=O\del{C_\kappa/\|\hw\|^2}$.
\end{enumerate}
This finishes the proof of \Cref{fact:log_ub} for $\hw$, in light of
\citep[Claim 3.4]{diakonikolas_adv_noise} which is stated below.
\begin{lemma}\label{fact:angle_01}
  Under \Cref{cond:well},
  \begin{align*}
    \cR_{0-1}(\hw)-\cR_{0-1}(\barw)\le\pr\del{\sign\del{\langle\hw,x\rangle}\ne\sign\del{\langle\barw,x\rangle}}\le2U\varphi(\hw,\barw).
  \end{align*}
\end{lemma}

\subsection{Recovering the general $\sqrt{\opt}$ bound}\label{sec:log_ub_add}

\citet{frei_soft_margin} showed an $\widetilde{O}\del[1]{\sqrt{\opt}}$ upper
bound under the ``soft-margin'' and ``sub-exponential'' conditions.
Here we give an alternative proof of this result using our proof technique. The result in this section will later serve as a guarantee of the first phase of our two-phase algorithm (cf. \Cref{sec:hinge}) that achieves
$\widetilde{O}(\opt)$ risk.

Recall that the only place we need \Cref{cond:rad} is in the proof of
\Cref{fact:rotate_diff_main}.
However, even without \Cref{cond:rad}, we can still prove the following general
bound which only needs \Cref{cond:well}.
\begin{lemma}\label{fact:rotate_diff_well}
  Under \Cref{cond:well}, for $\ell=\llog$,
  \begin{align*}
    \envert{\textup{term~\eqref{eq:log_diff_approx_equiv_2}}}\le \frac{12U}{\|\hw\|}.
  \end{align*}
\end{lemma}

Now with \Cref{fact:rotate_diff_well}, we can prove a weaker but more general
version of \Cref{fact:log_ub} (cf. \Cref{fact:log_ub_well}).
Further invoking \Cref{fact:pgd_wt_baru_main,fact:wt_norm_lb_main}
(cf. \Cref{fact:pgd_wt_baru,fact:wt_norm_lb} for the corresponding
sub-exponential results), and let $\epsopt=\sqrt{\epsilon}$, we can show the next
result.
We present the bound in terms of the angle instead of zero-one risk for later
application in \Cref{sec:hinge}.
\begin{lemma}\label{fact:log_ub_opt_well}
  Given the target error $\epsilon\in(0,1)$ and the failure probability
  $\delta\in(0,1/e)$, consider projected gradient descent \cref{eq:pgd}.
  If $\|x\|\le B$ almost surely,
  then with $\eta=4/B^2$, using $O\del{\frac{(B+1)^2\ln(1/\delta)}{\epsilon^2}}$ samples
  and $O\del{\frac{B^2}{\epsilon^{3/2}}}$ iterations, with probability
  $1-\delta$, projected gradient descent outputs $w_t$ with
  \begin{align*}
    \varphi(w_t,\baru)=O\del[1]{\sqrt{\opt+\epsilon}}.
  \end{align*}

  On the other hand, if $P_x$ is $(\alpha_1,\alpha_2)$-sub-exponential,
  then with
  $\eta=\widetilde{\Theta}(1/d)$, using
  $\widetilde{O}\del{\frac{d\ln(1/\delta)^3}{\epsilon^2}}$ samples and
  $\widetilde{O}\del{\frac{d\ln(1/\delta)^2}{\epsilon^{3/2}}}$ iterations, with
  probability $1-\delta$, projected gradient descent outputs $w_t$ with
  \begin{align*}
    \varphi(w_t,\baru)=O\del[1]{\sqrt{\opt\cdot\ln(1/\opt)+\epsilon}}.
  \end{align*}
\end{lemma}

The proofs of the results above are given in \Cref{app_sec:log_ub_add}.

\section{An $\widetilde{O}(\opt)$ upper bound with hinge loss}\label{sec:hinge}

We now show how to avoid \Cref{cond:rad} and achieve an $\widetilde{O}(\opt)$ zero-one risk bound using an extra step of hinge loss minimization. The key observation here is that the only place where \Cref{cond:rad} is used is in \Cref{fact:rotate_diff_main} for bounding term~\eqref{eq:log_diff_approx_equiv_2} for logistic loss. However, as noted in \Cref{fact:rotate_diff_main}, for hinge loss, term~\eqref{eq:log_diff_approx_equiv_2} is conveniently $0$. So a version of \Cref{fact:log_ub} holds for hinge loss, without using \Cref{cond:rad}, and dropping the third term of $\tfrac{C_\kappa}{\|\hw\|^2}$ in the max. Thus, to get an $\widetilde{O}(\opt)$ upper bound, we need to minimize the hinge loss to find a solution $\hw$ such that $\|\hw\| = \Omega(1)$ and $\cR_h(\hw)\le\cR_h(\|\hw\|\baru)+\epsopt$ for some $\epsopt = \widetilde{O}((\opt+\epsilon)^2)$. Unfortunately the requirement $\|\hw\| = \Omega(1)$ is non-convex. However, we can bypass the non-convexity by leveraging the solution of the logistic regression problem, which is guaranteed to make an angle of at most $\widetilde{O}(\sqrt{\opt+\epsilon})$ with $\baru$, even without \Cref{cond:rad}, by \Cref{fact:log_ub_opt_well}. This solution, represented by a unit vector $v$, gives us a ``warm start'' for hinge loss minimization. Specifically, suppose we optimize the hinge loss over the halfspace
\begin{equation} \label{eq:hinge-loss-domain}
  \cD:=\cbr{w\in\R^d\middle|\langle w,v\rangle\ge1},
\end{equation}
then any solution we find must have norm at least $1$. Furthermore, using the fact that $\varphi(v,\baru) \leq \widetilde{O}(\sqrt{\opt+\epsilon})$ and the positive homogeneity of the hinge loss, we can also conclude that the optimizer of the hinge loss satisfies $\cR_h(\hw)\le\cR_h(\|\hw\|\baru)+\epsopt$, giving us the desired solution.

While the above analysis does yield a simple two-phase polynomial time algorithm for getting an $\widetilde{O}(\opt)$ zero-one risk bound, closer analysis reveals a sample complexity requirement of $\widetilde{O}(1/\epsilon^4)$. We can improve the sample complexity requirement to $\widetilde{O}(1/\epsilon^2)$ by doing a custom analysis of SGD on the hinge loss (aka perceptron, \citep{novikoff_perceptron}) inspired by the above considerations. Thus we get the following two-phase algorithm\footnote{Note that the parameters $\eta$, $T$, etc. in this section are all chosen for the second phase.}:
\begin{enumerate}
  \item Run projected gradient descent under the settings of
  \Cref{fact:log_ub_opt_well}, and find a unit vector $v$ such that
  $\varphi(v,\baru)$ is $O\del[1]{\sqrt{\opt+\epsilon}}$ for bounded
  distributions, or $O\del[1]{\sqrt{\opt\cdot\ln(1/\opt)+\epsilon}}$ for
  sub-exponential distributions.

  \item
  Run projected SGD over the domain $\cD$ defined in \cref{eq:hinge-loss-domain} starting from $w_0 := v$: at step $t$, we sample
  $(x_t,y_t)\sim P$, and let
  \begin{align}\label{eq:sgd_hinge}
    w_{t+1}:=\Pi_{\cD}\sbr{w_t-\eta\ell'_h\del{y_t \langle w_t,x_t\rangle}y_tx_t}.
  \end{align}
  where we make the convention that $\ell'_h(0)=-1$.
\end{enumerate}

We show the following result on the expectation; it can also be turned into a
high-probability bound by probability amplification by repetition.
\begin{theorem}\label{fact:hinge_sgd}
  Given the target error $\epsilon\in(0,1/e)$, suppose \Cref{cond:well} holds.
  \begin{enumerate}
    \item First, for bounded distributions, with $\eta=\Theta(\epsilon)$, for
    all $T=\Omega(1/\epsilon^2)$,
    \begin{align*}
      \bbE\sbr{\min_{0\le t<T}\cR_{0-1}(w_t)}=O(\opt+\epsilon).
    \end{align*}

    \item On the other hand, for sub-exponential distributions, with
    $\eta=\Theta\del{\frac{\epsilon}{d\ln(d/\epsilon)^2}}$, for all
    $T=\Omega\del{\frac{d\ln(d/\epsilon)^2}{\epsilon^2}}$,
    \begin{align*}
      \bbE\sbr{\min_{0\le t<T}\cR_{0-1}(w_t)}=O(\opt\cdot\ln(1/\opt)+\epsilon).
    \end{align*}
  \end{enumerate}
\end{theorem}

\subsection{Proof of \Cref{fact:hinge_sgd}}

Here we give a proof sketch of \Cref{fact:hinge_sgd}, and we also focus on
bounded distributions for simplicity.
The full proof is given in \Cref{app_sec:hinge}.

Let $\barr:=1/\langle v,\baru\rangle$, and thus $\barr\baru\in\cD$.
At step $t$, we have
\begin{equation}\label{eq:sgd_tmp1}
\begin{split}
  \|w_{t+1}-\barr\baru\|^2\le\|w_t-\barr\baru\|^2-2\eta\ip{\ell'_h\del{y_t \langle w_t,x_t\rangle}y_tx_t}{w_t-\barr\baru}+\eta^2\ell'_h\del{y_t \langle w_t,x_t\rangle}^2\|x_t\|^2.
\end{split}
\end{equation}
Define
\begin{align*}
  \cM(w):=\bbE_{(x,y)\sim P}\sbr{-\ell'_h\del{y \langle w,x\rangle}}=\cR_{0-1}(w).
\end{align*}
Taking expectation of \cref{eq:sgd_tmp1} w.r.t. $(x_t,y_t)$, and note that
$\|x\|\le B$ almost surely and $(\ell'_h)^2=-\ell'_h$, we have
\begin{align}
  \bbE\sbr{\|w_{t+1}-\barr\baru\|^2}-\|w_t-\barr\baru\|^2 & \le-2\eta\ip{\nR_h(w_t)}{w_t-\barr\baru}+\eta^2B^2\cM(w_t) \nonumber \\
   & \le-2\eta\del{\cR_h(w_t)-\cR_h(\barr\baru)}+\eta^2B^2\cM(w_t). \label{eq:sgd_tmp_main}
\end{align}

To continue, we note the following lemma, which follows from
\Cref{fact:noisy_label_main,fact:ground_truth_diff,fact:rotate_diff_main}, and the
homogeneity of the hinge loss $\ell_h$.
\begin{lemma}\label{fact:hinge_ub_main}
  Suppose \Cref{cond:well} holds.
  Consider an arbitrary $w\in\cD$, and let $\varphi$ denote $\varphi(w,\baru)$.
  If $\|x\|\le B$ almost surely, then
  \begin{align*}
    \cR_h(\barr\baru)\le\cR_h(\|w\|\baru)+O\del[1]{(\opt+\epsilon)^2}
  \end{align*}
  and
  \begin{align*}
    \cR_h(w)-\cR_h(\|w\|\baru)\ge \frac{4R^3}{3U\pi^2}\|w\|\varphi^2-B\|w\|\varphi\cdot\opt.
  \end{align*}
\end{lemma}

The remaining steps of the proof proceed as follows.
We will prove the following: for $\varphi_t:=\varphi(w_t,\baru)$,
\begin{align}\label{eq:sgd_target_main}
  \bbE\sbr{\min_{0\le t\le T}\varphi_t}=O(\opt+\epsilon).
\end{align}
First, note that if $\varphi_t=O(\opt)$ for some $t$, then
\cref{eq:sgd_target_main} holds vacuously.
Hence, we assume that $\varphi_t\geq C\cdot\opt$ for all $t$ for a sufficiently
large constant $C>0$.
Then \Cref{fact:hinge_ub_main} ensures that for some constant $C_1$,
\begin{align*}
  \cR_h(w_t)-\cR_h(\|w_t\|\baru)\ge C_1\|w_t\|\varphi_t^2\ge C_1\varphi_t^2,
\end{align*}
where we also use $\|w\|\ge1$ for all $w\in\cD$.

Next, note that $\cM(w_t)=O(\varphi_t)$, due to our assumption
$\varphi_t\geq C\cdot\opt$ and \Cref{fact:angle_01}.
If $\varphi_t\le\epsilon$, then \cref{eq:sgd_target_main} also holds, otherwise
we can assume $\epsilon\le\varphi_t$, and let $\eta=C_2\epsilon$ for some small
enough constant $C_2$, such that
\begin{align*}
  \eta B^2\cM(w_t)\le C_1\epsilon\varphi_t\le C_1\varphi_t^2.
\end{align*}

Now \cref{eq:sgd_tmp_main} implies
\begin{align*}
  \bbE\sbr{\|w_{t+1}-\barr\baru\|^2}-\|w_t-\barr\baru\|^2 & \le-2\eta C_1\varphi_t^2-\eta\cdot O\del[1]{(\opt+\epsilon)^2}+\eta C_1\varphi_t^2 \\
   & =-\eta C_1\varphi_t^2-\eta\cdot O\del[1]{(\opt+\epsilon)^2}.
\end{align*}
Taking the total expectation and telescoping the above inequality for all $t$, we have
\begin{align*}
  \bbE\sbr{\frac{1}{T}\sum_{t<T}^{}\varphi_t^2}\le \frac{\|w_0-\barr\baru\|^2}{\eta C_1T}+O\del[1]{(\opt+\epsilon)^2}.
\end{align*}
Recall that
\begin{align*}
  \|w_0-\barr\baru\|=\|v-\barr\baru\|=O\del[1]{\sqrt{\opt+\epsilon}}
\end{align*}
due to the first phase of the algorithm.
Since $\eta=C_2\epsilon$, we
can further let $T=\Omega(1/\epsilon^2)$ and finish the proof.

\section{Open problems}

Here are some open problems.
First, as shown by \Cref{fact:hinge_sgd}, we can achieve
$O\del{\opt\cdot\ln(1/\opt)+\epsilon}$ zero-one risk using the two-phase algorithm.
However, previous algorithms can reach $O(\opt+\epsilon)$
\citep{awasthi_localization,diakonikolas_adv_noise}.
Is it possible to develop an algorithm relying on a small constant number of convex optimization phases that achieves $O(\opt+\epsilon)$ risk?

It is also interesting to consider neural networks.
Previously, \citet{frei_adv_noise} showed that stochastic gradient descent
on a two-layer leaky ReLU network of any width achieves
$\widetilde{O}\del[1]{\sqrt{\opt}}$ zero-one risk, where $\opt$ still denotes
the best zero-one risk of a linear classifier.
On the other hand, \citet{early_stop} showed that a wide two-layer ReLU network can
even achieve the optimal Bayes risk, but their required width depends on a complexity measure that may be exponentially large in the worst case.
Can a network with a reasonable width always reach a zero-one risk of $O(\opt)$?

\subsection*{Acknowledgements}

Ziwei Ji thanks Matus Telgarsky for helpful discussions, and the NSF for support
under grant IIS-1750051.

\bibliography{bib}
\bibliographystyle{plainnat}

\newpage
\appendix

\section{Technical lemmas}

Here are some technical results we will need in our analysis.

\begin{lemma}\label{fact:risk_ub_lb}
  Let $r,\rho>0$ be given, then
  \begin{align*}
    \frac{2}{\rho}(1-e^{-r\rho})\le\int_0^{2\pi}\llog\del{r\rho\envert{\cos(\theta)}}r\dif\theta\le \frac{8\sqrt{2}}{\rho}.
  \end{align*}
\end{lemma}
\begin{proof}
  First note that by symmetry,
  \begin{align*}
    \int_0^{2\pi}\llog\del{r\rho\envert{\cos(\theta)}}r\dif\theta=4\int_0^{\frac{\pi}{2}}\llog\del{r\rho\cos(\theta)}r\dif\theta.
  \end{align*}
  On the upper bound, note that $\llog\del{r\rho\cos(\theta)}$ is increasing as
  $\theta$ goes from $0$ to $\frac{\pi}{2}$, and moreover
  $\sin(\theta)\ge \frac{\sqrt{2}}{2}$ for
  $\theta\in\del{\frac{\pi}{4},\frac{\pi}{2}}$, therefore
  \begin{align*}
    4\int_0^{\frac{\pi}{2}}\llog\del{r\rho\cos(\theta)}r\dif\theta\le8\int_{\frac{\pi}{4}}^{\frac{\pi}{2}}\llog\del{r\rho\cos(\theta)}r\dif\theta\le \frac{8\sqrt{2}}{\rho}\int_{\frac{\pi}{4}}^{\frac{\pi}{2}}\llog\del{r\rho\cos(\theta)}r\rho\sin(\theta)\dif\theta.
  \end{align*}
  Also because $\llog(z)\le\exp(-z)$,
  \begin{align*}
    \int_0^{2\pi}\llog\del{r\rho\envert{\cos(\theta)}}r\dif\theta & \le \frac{8\sqrt{2}}{\rho}\int_{\frac{\pi}{4}}^{\frac{\pi}{2}}\exp\del{-r\rho\cos(\theta)}r\rho\sin(\theta)\dif\theta \\
     & =\frac{8\sqrt{2}}{\rho}\del{1-\exp\del{-\frac{\sqrt{2}r\rho}{2}}} \\
     & \le \frac{8\sqrt{2}}{\rho}.
  \end{align*}

  On the lower bound, note that $\llog(z)\ge \frac{1}{2}\exp(-z)$ for $z\ge0$,
  therefore
  \begin{align*}
    \int_0^{2\pi}\llog\del{r\rho\envert{\cos(\theta)}}r\dif\theta=4\int_0^{\frac{\pi}{2}}\llog\del{r\rho\cos(\theta)}r\dif\theta & \ge2\int_0^{\frac{\pi}{2}}\exp\del{-r\rho\cos(\theta)}r\dif\theta \\
     & \ge \frac{2}{\rho}\int_0^{\frac{\pi}{2}}\exp\del{-r\rho\cos(\theta)}r\rho\sin(\theta)\dif\theta \\
     & =\frac{2}{\rho}\del{1-e^{-r\rho}}.
  \end{align*}
\end{proof}

\begin{lemma}\label{fact:opt_ip_bound}
  Given $w,w'\in\R^d$, suppose
  $\pr_{(x,y)\sim P}\del{y\ne\sign\del{\langle w,x\rangle}}=\opt$.
  If $\|x\|\le B$ almost surely, then
  \begin{align*}
    \bbE_{(x,y)\sim P}\sbr{\1_{y\ne\sign\del{\langle w,x\rangle}}\envert{\langle w',x\rangle}}\le B\|w'\|\cdot\opt.
  \end{align*}
  If $P_x$ is $(\alpha_1,\alpha_2)$-sub-exponential, and $\opt\le \frac{1}{e}$,
  then
  \begin{align*}
    \bbE_{(x,y)\sim P}\sbr{\1_{y\ne\sign\del{\langle w,x\rangle}}\envert{\langle w',x\rangle}}\le(1+2\alpha_1)\alpha_2\|w'\|\cdot\opt\cdot\ln\del{\frac{1}{\opt}}.
  \end{align*}
\end{lemma}
\begin{proof}
  If $\|x\|\le B$ almost surely, then
  \begin{align*}
    \bbE_{(x,y)\sim P}\sbr{\1_{y\ne\sign\del{\langle w,x\rangle}}\envert{\langle w',x\rangle}}\le B\|w'\|\bbE_{(x,y)\sim P}\sbr{\1_{y\ne\sign\del{\langle w,x\rangle}}}=B\|w'\|\cdot\opt.
  \end{align*}
  Below we assume $P_x$ is $(\alpha_1,\alpha_2)$-sub-exponential.

  Let $\nu_x:=\langle w',x\rangle$; we first give some tail bounds for $\nu_x$.
  Since $P_x$ is $(\alpha_1,\alpha_2)$-sub-exponential, for any $t>0$, we have
  \begin{align*}
    \pr\del{\envert{\ip{\frac{w'}{\|w'\|}}{x}}\ge t}\le\alpha_1\exp\del{-\frac{t}{\alpha_2}},\quad\textup{equivalently}\quad\pr\del{|\nu_x|\ge t}\le\alpha_1\exp\del{-\frac{t}{\alpha_2\|w'\|}}.
  \end{align*}
  Let $\mu(t):=\pr\del{|\nu_x|\ge t}$.
  Given any threshold $\tau>0$, integration by parts gives
  \begin{align}\label{eq:log_diff_approx_err_tail}
    \bbE\sbr{\1_{|\nu_x|\ge\tau}|\nu_x|}=\int_\tau^\infty t\cdot\del{-\dif \mu(t)}=\tau \mu(\tau)+\int_\tau^\infty \mu(t)\dif t\le\alpha_1\del{\alpha_2\|w'\|+\tau}\exp\del{-\frac{\tau}{\alpha_2\|w'\|}}.
  \end{align}

  Now let $\tau:=\alpha_2\|w'\|\ln\del{\frac{1}{\opt}}$.
  Note that
  \begin{align*}
    \bbE_{(x,y)\sim P}\sbr{\1_{y\ne\sign\del{\langle w,x\rangle}}\envert{\langle w',x\rangle}}=\bbE_{(x,y)\sim P}\sbr{\1_{|\nu_x|\le\tau}\1_{y\ne\sign\del{\langle w,x\rangle}}|\nu_x|}+\bbE_{(x,y)\sim P}\sbr{\1_{|\nu_x|\ge\tau}\1_{y\ne\sign\del{\langle w,x\rangle}}|\nu_x|}.
  \end{align*}
  We bound the two parts separately.
  When $|\nu_x|\le\tau$, we have
  \begin{align*}
    \bbE\sbr{\1_{|\nu_x|\le\tau}\1_{y\ne\sign\del{\langle w,x\rangle}}|\nu_x|}\le\tau\bbE\sbr{\1_{y\ne\sign\del{\langle w,x\rangle}}}=\tau\cdot\opt=\alpha_2\|w'\|\cdot\opt\cdot\ln\del{\frac{1}{\opt}}.
  \end{align*}
  On the other hand, when $|\nu_x|\ge\tau$, \cref{eq:log_diff_approx_err_tail}
  gives
  \begin{align*}
    \bbE_{(x,y)\sim P}\sbr{\1_{|\nu_x|\ge\tau}\1_{y\ne\sign\del{\langle w,x\rangle}}|\nu_x|} & \le\bbE\sbr{\1_{|\nu_x|\ge\tau}|\nu_x|} \\
     & \le\alpha_1\alpha_2\|w'\|\del{1+\ln\del{\frac{1}{\opt}}}\opt \\
     & \le2\alpha_1\alpha_2\|w'\|\cdot\opt\cdot\ln\del{\frac{1}{\opt}},
  \end{align*}
  where we also use $\opt\le \frac{1}{e}$.
  To sum up,
  \begin{align*}
    \bbE_{(x,y)\sim P}\sbr{\1_{y\ne\sign\del{\langle w,x\rangle}}\envert{\langle w',x\rangle}}\le(1+2\alpha_1)\alpha_2\|w'\|\cdot\opt\cdot\ln\del{\frac{1}{\opt}}.
  \end{align*}
\end{proof}

\section{Omitted proofs from \Cref{sec:log_lb}}\label{app_sec:log_lb}

In this section, we will prove \Cref{fact:log_lb}.
First, we bound the density and support of $\lbu_x$.
\begin{lemma}\label{fact:log_lb_density}
  If $\opt\le \frac{1}{100}$, then it holds that $\lbl_3\le \frac{1}{15}$, and
  $\frac{1}{2\pi}\le\lbl_4\le \frac{1}{\pi}$.
  As a result, $\lbu_x$ is supported on $\cB(2):=\cbr{x\middle|\|x\|\le2}$ with
  its density bounded by $2$.
\end{lemma}
\begin{proof}
  For $\lbl_3$, we have
  \begin{align*}
    \lbl_3=\frac{2}{3}\sqrt{\opt}(1-\opt)\le \frac{2}{3}\sqrt{\opt}\le \frac{2}{3}\frac{1}{10}=\frac{1}{15}.
  \end{align*}
  For $\lbu_4$, its total measure can be bounded as below:
  \begin{align*}
    1-\opt-2\sqrt{\opt}-\lbl_3\ge1-\frac{1}{100}-\frac{2}{10}-\frac{1}{15}\ge \frac{1}{2},
  \end{align*}
  therefore $\lbl_4\ge \frac{1}{2\pi}$.
  The upper bound $\lbl_4\le \frac{1}{\pi}$ is trivial.

  On the support of $\lbu_x$, note that for $\lbu_1$, the largest $\ell_2$ norm
  is given by
  \begin{align*}
    1+\frac{\sqrt{2}}{2}\sqrt{\frac{\opt}{2}}\le1+\frac{1}{20}\le2.
  \end{align*}
  For $\lbu_2$, the largest $\ell_2$ norm can be bounded by
  \begin{align*}
    1+\sqrt{\opt}\le1+\frac{1}{10}\le2.
  \end{align*}
  For $\lbu_3$, the largest $\ell_2$ norm can be bounded by
  \begin{align*}
    1+\frac{\sqrt{2}}{2}\sqrt{\frac{\lbl_3}{2}}\le1+\frac{1}{2}\sqrt{\frac{1}{15}}\le2.
  \end{align*}
  Finally, it is easy to verify that if $\opt\le \frac{1}{100}$, then $\lbu_1$,
  $\lbu_2$ and $\lbu_3$ do not overlap, therefore the density of $\lbu$ is
  bounded by $1+\frac{1}{\pi}\le2$.
\end{proof}

Next we verify that $\lbu_x$ is isotropic up to a multiplicative factor.
We first note the following fact; its proof is straightforward and omitted.
\begin{lemma}\label{fact:square_int}
  It holds that
  \begin{align*}
    \int_{a-\frac{\delta}{2}}^{a+\frac{\delta}{2}}\int_{b-\frac{\delta}{2}}^{b+\frac{\delta}{2}}xy\dif y\dif x=ab\delta^2,\quad\textup{and}\quad\int_{a-\frac{\delta}{2}}^{a+\frac{\delta}{2}}\int_{b-\frac{\delta}{2}}^{b+\frac{\delta}{2}}(x^2-y^2)\dif y\dif x=(a^2-b^2)\delta^2.
  \end{align*}
\end{lemma}

Then we can prove the following result.
\begin{lemma}\label{fact:log_lb_isotropic}
  It holds that $\bbE_{x\sim\lbu_x}[x]=0$, and $\bbE_{x\sim\lbu_x}[x_1x_2]=0$,
  and $\bbE_{x\sim\lbu_x}\sbr[1]{x_1^2-x_2^2}=0$.
\end{lemma}
\begin{proof}
  It follows from the symmetry of $\lbu$ that $\bbE_{x\sim\lbu_x}[x]=0$.

  To verify $\bbE_{x\sim\lbu_x}[x_1x_2]=0$, note that the expectation of
  $x_1x_2$ is $0$ on $\lbu_3$ and $\lbu_4$, and thus we only need to check
  $\lbu_1$ and $\lbu_2$.
  First, due to \Cref{fact:square_int}, we have
  \begin{align*}
    \bbE_{(x,y)\sim\lbu_1}\sbr{x_1x_2}=-\frac{\opt}{2}.
  \end{align*}
  Additionally,
  \begin{align*}
    \bbE_{(x,y)\sim\lbu_2}\sbr{x_1x_2} & =2\int_0^{\sqrt{\opt}}\int_0^1x_1x_2\dif x_2\dif x_1=\frac{\opt}{2}.
  \end{align*}
  Therefore $\bbE_{x\sim\lbu_x}[x_1x_2]=0$.

  Finally, note that the expectation of $x_1^2-x_2^2$ is $0$ on $\lbu_1$ due to
  \Cref{fact:square_int}, and also $0$ on $\lbu_4$ due to symmetry; therefore we
  only need to consider $\lbu_2$ and $\lbu_3$.
  We have
  \begin{align*}
    \bbE_{(x,y)\sim\lbu_2}\sbr{x_1^2-x_2^2} & =2\int_0^{\sqrt{\opt}}\int_0^1\del[1]{x_1^2-x_2^2}\dif x_2\dif x_1=\frac{2}{3}\opt^{3/2}-\frac{2}{3}\sqrt{\opt}=-\lbl_3.
  \end{align*}
  Since $\bbE_{(x,y)\sim\lbu_3}\sbr[1]{x_1^2-x_2^2}=\lbl_3$ by
  \Cref{fact:square_int}, it follows that
  $\bbE_{x\sim\lbu_x}\sbr[1]{x_1^2-x_2^2}=0$.
\end{proof}

Next, we give a proof of the risk lower bound of \Cref{fact:log_lb}.
For simplicity, in this section we will let $\cR$ denote $\cRlog$.
For $i=1,2,3,4$, we also let
$\cR_i(w):=\bbE_{(x,y)\sim\lbu_i}\sbr[1]{\llog\del{y \langle w,x\rangle}}$;
therefore $\cR(w):=\sum_{i=1}^{4}\cR_i(w)$.
We first prove \Cref{fact:log_lb_ref_sol}, showing that there exists a solution
$\barw$ with $\|\barw\|=\Theta\del{\frac{1}{\sqrt{\opt}}}$ and
$\cR(\barw)=O\del[1]{\sqrt{\opt}}$.

\begin{proof}[Proof of \Cref{fact:log_lb_ref_sol}]
  We consider $\cR_1$, $\cR_2$, $\cR_3$ and $\cR_4$ respectively.
  \begin{enumerate}
    \item For $\lbu_1$, note that the minimum of
    $y \langle\barw,x\rangle$ is
    \begin{align*}
      -\del{\frac{\sqrt{2}}{2}+\frac{1}{2}\sqrt{\frac{\opt}{2}}}\barr=-\frac{3\sqrt{2}}{2}\frac{1}{\sqrt{\opt}}-\frac{3\sqrt{2}}{4}.
    \end{align*}
    Because $\llog(z)\le-z+1$ when $z\le0$, and $\opt\le \frac{1}{100}$, we have
    \begin{align*}
      \cR_1(\barw)\le\llog\del{-\frac{3\sqrt{2}}{2}\frac{1}{\sqrt{\opt}}-\frac{3\sqrt{2}}{4}}\cdot\opt & \le \frac{3\sqrt{2}}{2}\sqrt{\opt}+\del{\frac{3\sqrt{2}}{4}+1}\opt \\
       & \le \frac{3\sqrt{2}}{2}\sqrt{\opt}+\del{\frac{3\sqrt{2}}{4}+1}\frac{1}{10}\sqrt{\opt} \\
       & \le \frac{5\sqrt{\opt}}{2}.
    \end{align*}

    \item For $\lbu_2$, we have
    \begin{align*}
      \cR_2(\barw)=2\int_0^{\sqrt{\opt}}\int_0^1\llog(x_1\barr)\dif x_2\dif x_1 & =2\int_0^{\sqrt{\opt}}\llog(x_1\barr)\dif x_1 \\
       & \le2\int_0^{\sqrt{\opt}}\exp(-x_1\barr)\dif x_1 \\
       & =\frac{2}{\barr}\del{1-\exp\del{-\barr\sqrt{\opt}}}\le \frac{2}{\barr},
    \end{align*}
    where we use $\llog(z)\le\exp(-z)$.

    \item For $\lbu_3$, the minimum of $y \langle\barw,x\rangle$ is
    \begin{align*}
      \del{1-\frac{1}{2}\sqrt{\frac{\lbl_3}{2}}}\barr\ge \frac{2\barr}{3},
    \end{align*}
    where we use $\lbl_3\le \frac{1}{15}$ by \Cref{fact:log_lb_density}.
    Further note that $\llog(z)\le1/z$ when $z>0$, we have
    \begin{align*}
      \cR_3(\barw)\le\lbl_3\llog\del{\frac{2\barr}{3}}\le \frac{1/15}{2\barr/3}\le \frac{1}{10\barr}.
    \end{align*}

    \item For $\lbu_4$,
    \begin{align*}
      \cR_4(\barw)=\int_0^1\int_0^{2\pi}\llog\del{r\barr\envert{\cos(\theta)}}\lbl_4r\dif\theta\dif r\le \frac{1}{\pi}\int_0^1\int_0^{2\pi}\llog\del{r\barr\envert{\cos(\theta)}}r\dif\theta\dif r,
    \end{align*}
    where we use $\lbl_4\le \frac{1}{\pi}$ from \Cref{fact:log_lb_density}.
    \Cref{fact:risk_ub_lb} then implies
    \begin{align*}
      \cR_4(\barw)\le \frac{1}{\pi}\int_0^1\frac{8\sqrt{2}}{\barr}\dif r=\frac{8\sqrt{2}}{\pi\barr}.
    \end{align*}
  \end{enumerate}

  Putting everything together, we have
  \begin{align*}
    \cR(\barw) & =\cR_1(\barw)+\cR_2(\barw)+\cR_3(\barw)+\cR_4(\barw) \\
     & \le \frac{5\sqrt{\opt}}{2}+\frac{2}{\barr}+\frac{1}{10\barr}+\frac{8\sqrt{2}}{\pi\barr} \\
     & \le \frac{5\sqrt{\opt}}{2}+\frac{6}{\barr}\le5\sqrt{\opt}.
  \end{align*}
\end{proof}

Next we prove \Cref{fact:w*_norm_ub}, the upper bound on $\|w^*\|$.
\begin{proof}[Proof of \Cref{fact:w*_norm_ub}]
  Let
  \begin{align*}
    u:=\del{\frac{\sqrt{2}}{2},-\frac{\sqrt{2}}{2}},\quad\textup{and}\quad v:=\del{\frac{\sqrt{2}}{2}-\frac{1}{2}\sqrt{\frac{\opt}{2}},-\frac{\sqrt{2}}{2}-\frac{1}{2}\sqrt{\frac{\opt}{2}}}.
  \end{align*}
  Let $\phi$ denote the angle between $u$ and $v$, then
  \begin{align*}
    \phi\le\tan(\phi)=\frac{\sqrt{2}}{2}\sqrt{\frac{\opt}{2}}=\frac{\sqrt{\opt}}{2}\le \frac{1}{20}\le \frac{\pi}{24},
  \end{align*}
  and it follows that the angle between $v$ and $w^*$ is bounded by
  \begin{align*}
    \frac{\pi}{24}+\frac{\pi}{4}+\frac{\sqrt{\opt}}{30}\le\frac{\pi}{24}+\frac{\pi}{4}+\frac{\pi}{24}=\frac{\pi}{3}.
  \end{align*}
  Moreover, note that the maximum of $y \langle w^*,x\rangle$ on $\lbu_1$ is
  given by
  \begin{align*}
    -\langle w^*,v\rangle\le-r^*\|v\|\cos\del{\frac{\pi}{3}}\le-r^*\cos\del{\frac{\pi}{3}}=-\frac{r^*}{2}.
  \end{align*}
  Additionally because $\llog(z)>-z$, we have
  \begin{align*}
    \cR(w^*)\ge\cR_1(w^*)\ge\llog\del{-\frac{r^*}{2}}\cdot\opt>\frac{r^*}{2}\cdot\opt.
  \end{align*}
  If $r^*>\frac{10}{\sqrt{\opt}}$, then $\cR(w^*)>5\sqrt{\opt}$, which
  contradicts the definition of $w^*$ in light of \Cref{fact:log_lb_ref_sol}.
  Therefore $r^*\le \frac{10}{\sqrt{\opt}}$.
\end{proof}

Next we prove \Cref{fact:nonzero_grad}.
\begin{proof}[Proof of \Cref{fact:nonzero_grad}]
  Let $w=(r,\theta)$, where $0\le r\le \frac{10}{\sqrt{\opt}}$ and
  $0\le\theta\le \frac{\sqrt{\opt}}{30}$.
  We will consider the projection of $\nR(w)$ onto the
  direction $e_2:=(0,1)$, and show that this projection cannot be zero.
  \begin{enumerate}
    \item For $\lbu_1$, the gradient of this part has a negative inner product
    with $e_2$, due to the construction of $\lbu_1$ and the fact $\llog'<0$.

    \item For $\lbu_2$, the inner product between $e_2$ and the gradient of this
    part is given by
    \begin{align}\label{eq:grad_2}
      2\int_0^{\sqrt{\opt}}\int_0^1\llog'(x_1w_1+x_2w_2)x_2\dif x_2\dif x_1.
    \end{align}
    Note that $x_1w_1\le rx_1$, while
    \begin{align*}
      x_2w_2=x_2r\sin\del{\theta}\le r\theta\le \frac{10}{\sqrt{\opt}}\frac{\sqrt{\opt}}{30}=\frac{1}{3},
    \end{align*}
    and that $\llog'$ is increasing, therefore
    \begin{align*}
      \llog'(x_1w_1+x_2w_2)\le\llog'\del{rx_1+\frac{1}{3}}.
    \end{align*}
    We can then upper bound \cref{eq:grad_2} as follows:
    \begin{align*}
      \textrm{\cref{eq:grad_2}} & \le2\int_0^{\sqrt{\opt}}\int_0^1\llog'\del{rx_1+\frac{1}{3}}x_2\dif x_2\dif x_1 \\
       & =\int_0^{\sqrt{\opt}}\llog'\del{rx_1+\frac{1}{3}}\dif x_1 \\
       & =\frac{1}{r}\del{\llog\del{\frac{1}{3}+r\sqrt{\opt}}-\llog\del{\frac{1}{3}}}.
    \end{align*}
    Now we consider two cases.
    If $r\sqrt{\opt}\le2$, then it follows from the convexity of $\llog$ that
    \begin{align*}
      \textrm{\cref{eq:grad_2}}\le \frac{1}{r}\llog'\del{\frac{1}{3}+r\sqrt{\opt}}r\sqrt{\opt}\le\llog'(3)\sqrt{\opt}\le-\frac{\sqrt{\opt}}{30}.
    \end{align*}
    On the other hand, if $r\sqrt{\opt}\ge2$, then
    \begin{align*}
      \textrm{\cref{eq:grad_2}}\le \frac{1}{r}\del{\llog\del{\frac{7}{3}}-\llog\del{\frac{1}{3}}}\le \frac{\sqrt{\opt}}{10}\del{\llog\del{\frac{7}{3}}-\llog\del{\frac{1}{3}}}\le-\frac{\sqrt{\opt}}{30}.
    \end{align*}
    Therefore, it always holds that
    $\textrm{\cref{eq:grad_2}}\le-\frac{\sqrt{\opt}}{30}$.

    \item For $\lbu_3$, the gradient of this part can have a positive inner
    product with $e_2$.
    For simplicity, let $\rho:=\frac{1}{2}\sqrt{\frac{\lbl_3}{2}}$.
    To upper bound this inner product, it is enough to consider the region given
    by
    \begin{align*}
      \del{[1-\rho,1+\rho]\times[-\rho,0]}\cup\del{[-1-\rho,-1+\rho]\times[0,\rho]}.
    \end{align*}
    Moreover, note that $y \langle w,x\rangle\ge0$ on $\lbu_3$, therefore
    $\llog'\del{y \langle w,x\rangle}\ge-\frac{1}{2}$.
    Therefore the inner product between $e_2$ and the gradient of $\lbu_3$ can
    be upper bounded by (note that $x_2\le0$ in the integral)
    \begin{align*}
      2\int_{1-\rho}^{1+\rho}\int_{-\rho}^0-\frac{1}{2}x_2\dif x_2\dif x_1=\rho^3=\frac{\sqrt{\lbl_3}}{16\sqrt{2}}\lbl_3\le\frac{\sqrt{1/15}}{16\sqrt{2}}\frac{2}{3}\sqrt{\opt}<\frac{\sqrt{\opt}}{60}.
    \end{align*}
    where we use $\lbl_3\le \frac{1}{15}$ by \Cref{fact:log_lb_density} and
    $\lbl_3\le\frac{2}{3}\sqrt{\opt}$ by its definition.

    \item For $\lbu_4$, we further consider two cases.
    \begin{enumerate}
      \item Consider the part of $\lbu_4$ with polar angles in
      $(-\frac{\pi}{2}+2\theta,\frac{\pi}{2})\cup(\frac{\pi}{2}+2\theta,\frac{3\pi}{2})$.
      By symmetry, the gradient of this part is along the direction with polar
      angle $\pi+\theta$, and it has a negative inner product with $e_2$.

      \item Consider the part of $\lbu_4$ with polar angles in
      $(-\frac{\pi}{2},-\frac{\pi}{2}+2\theta)\cup(\frac{\pi}{2},\frac{\pi}{2}+2\theta)$.
      We can verify that the gradient of this part has a positive inner product
      with $e_2$; moreover, since $-1<\llog'<0$, this inner product can be upper
      bounded by
      \begin{align*}
        2\int_0^1\int_0^{2\theta}r'\cos(\theta')\lbl_4r'\dif\theta'\dif r'=2\lbl_4\cdot \frac{1}{3}\cdot\sin(2\theta)\le\frac{4\theta}{3\pi}\le\frac{4}{3\pi}\frac{\sqrt{\opt}}{30}<\frac{\sqrt{\opt}}{60},
      \end{align*}
      where we also use $\lbl_4\le \frac{1}{\pi}$ and $\sin(z)\le z$ for
      $z\ge0$.
    \end{enumerate}
  \end{enumerate}
  As a result, item 3 and item 4(b) cannot cancel item 2, and thus $\nR(w)$
  cannot be $0$.
\end{proof}

Now we are ready to prove the risk lower bound of \Cref{fact:log_lb}.
\begin{proof}[Proof of \Cref{fact:log_lb} risk lower bound]
  It is clear that $\cR$ has bounded sub-level sets, and therefore can be
  globally minimized.
  Let the polar coordinates of the global minimizer be given by
  $(r^*,\theta^*)$, where $|\theta^*|\le\pi$.
  Assume that $\theta^*\in\sbr{-\frac{\sqrt{\opt}}{30},\frac{\sqrt{\opt}}{30}}$;
  due to $\lbu_1$ and $\lbu_2$, it actually follows that
  $\theta^*\in\sbr{0,\frac{\sqrt{\opt}}{30}}$.
  \Cref{fact:w*_norm_ub} then implies $r^*\le \frac{10}{\sqrt{\opt}}$, and then
  \Cref{fact:nonzero_grad} implies $\nR(w^*)\ne0$, a contradiction.

  It then follows that $w^*$ is wrong on a $\frac{\theta^*}{\pi}$ portion of
  $\lbu_4$.
  Since the total measure of $\lbu_4$ is more than half due to
  \Cref{fact:log_lb_density}, we have
  \begin{align*}
    \cR_{0-1}(w^*)\ge \frac{1}{2}\frac{\theta^*}{\pi}\ge \frac{\sqrt{\opt}}{60\pi}.
  \end{align*}
\end{proof}

\section{Omitted proofs from \Cref{sec:log_ub}}\label{app_sec:log_ub}

In this section, we provide omitted proofs from \Cref{sec:log_ub}.
First, we prove some general results that will be used later.

\begin{lemma}\label{fact:Rlog_inv_norm}
  Under \Cref{cond:well}, for any $w\in\R^d$,
  \begin{align*}
    \bbE\sbr{\llog\del{\envert{\langle w,x\rangle}}}\le \frac{12U}{\|w\|}.
  \end{align*}
\end{lemma}
\begin{proof}
  Let $v$ denote an arbitrary vector orthogonal to $w$, and let $p$ denote the
  density of the projection of $P_x$ onto the space spanned by $w$ and $v$.
  Then we have
  \begin{align*}
    \bbE\sbr{\llog\del{\envert{\langle w,x\rangle}}}=\int_0^\infty\int_0^{2\pi}\llog\del{r\|w\|\envert{\cos(\theta)}}p(r,\theta)r\dif\theta\dif r.
  \end{align*}
  Invoking \Cref{cond:well}, we have
  \begin{align*}
    \bbE\sbr{\llog\del{\envert{\langle w,x\rangle}}}\le\int_0^\infty\sigma(r)\del{\int_0^{2\pi}\llog\del{r\|w\|\envert{\cos(\theta)}}r\dif\theta}\dif r.
  \end{align*}
  \Cref{fact:risk_ub_lb} then implies
  \begin{align*}
    \bbE\sbr{\llog\del{\envert{\langle w,x\rangle}}}\le\int_0^\infty\sigma(r)\frac{8\sqrt{2}}{\|w\|}\dif r.
  \end{align*}
  Then it follows from \Cref{cond:well} that
  \begin{align*}
    \bbE\sbr{\llog\del{\envert{\langle w,x\rangle}}}\le \frac{8\sqrt{2}U}{\|w\|}\le \frac{12U}{\|w\|}.
  \end{align*}
\end{proof}

Next, we note that following the direction of the ground-truth solution $\baru$
can achieve $\widetilde{O}\del[1]{\sqrt{\opt}}$ logistic risk.
\begin{lemma}\label{fact:log_ub_ref_sol}
  Given $\rho>0$, under \Cref{cond:well}, if $\|x\|\le B$ almost surely, then
  \begin{align*}
    \cRlog(\rho\baru)\le \frac{12U}{\rho}+\rho B\cdot\opt,\quad\textup{with}\quad\inf_{\rho>0}\cRlog(\rho\baru)\le\sqrt{50UB\cdot\opt},
  \end{align*}
  while if $P_x$ is $(\alpha_1,\alpha_2)$-sub-exponential, then
  \begin{align*}
    \cRlog(\rho\baru)\le\frac{12U}{\rho}+(1+2\alpha_1)\alpha_2\rho\cdot\opt\cdot\ln\del{\frac{1}{\opt}},
  \end{align*}
  with
  \begin{align*}
    \inf_{\rho>0}\cRlog(\rho\baru)\le\sqrt{50(1+2\alpha_1)\alpha_2U\cdot\opt\cdot\ln\del{\frac{1}{\opt}}}.
  \end{align*}
\end{lemma}
\begin{proof}
  Note that
  \begin{align*}
    \cRlog(\rho\baru) & =\bbE_{(x,y)\sim P}\sbr{\llog\del{y \langle\rho\baru,x\rangle}} \\
     & =\bbE_{x\sim P_x}\sbr{\llog\del{\envert{\langle\rho\baru,x\rangle}}}+\bbE_{(x,y)\sim P}\sbr{\llog\del{y \langle\rho\baru,x\rangle}-\llog\del{\envert{\langle\rho\baru,x\rangle}}}.
  \end{align*}
  Since $\llog(-z)-\llog(z)=z$, and also invoking \Cref{fact:Rlog_inv_norm}, we have
  \begin{align*}
    \cRlog(\rho\baru) & =\bbE_{x\sim P_x}\sbr{\llog\del{\envert{\langle\rho\baru,x\rangle}}}+\bbE_{(x,y)\sim P}\sbr{\1_{y\ne\sign\del{\langle\baru,x\rangle}}\cdot(-y)\langle\rho\baru,x\rangle} \\
     & \le \frac{12U}{\rho}+\bbE_{(x,y)\sim P}\sbr{\1_{y\ne\sign\del{\langle\baru,x\rangle}}\cdot(-y)\langle\rho\baru,x\rangle}.
  \end{align*}

  If $\|x\|\le B$ almost surely, then \Cref{fact:opt_ip_bound} further implies
  \begin{align*}
    \cRlog(\rho\baru)\le\frac{12U}{\rho}+\rho B\cdot\opt,
  \end{align*}
  and thus
  \begin{align*}
    \inf_{\rho>0}\cRlog(\rho\baru)\le2\sqrt{12UB\cdot\opt}\le\sqrt{50UB\cdot\opt}.
  \end{align*}

  If $P_x$ is $(\alpha_1,\alpha_2)$-sub-exponential, then
  \Cref{fact:opt_ip_bound} further implies
  \begin{align*}
    \cRlog(\rho\baru)\le\frac{12U}{\rho}+(1+2\alpha_1)\alpha_2\rho\cdot\opt\cdot\ln\del{\frac{1}{\opt}},
  \end{align*}
  and therefore
  \begin{align*}
    \inf_{\rho>0}\cRlog(\rho\baru)\le2\sqrt{12(1+2\alpha_1)\alpha_2 U\cdot\opt\cdot\ln\del{\frac{1}{\opt}}}\le\sqrt{50(1+2\alpha_1)\alpha_2U\cdot\opt\cdot\ln\del{\frac{1}{\opt}}}.
  \end{align*}
\end{proof}

Next we prove a risk lower bound, that will later be used to prove lower bounds
on $\|w^*\|$ and $\|w_t\|$.
\begin{lemma}\label{fact:cR_lb}
  Under \Cref{cond:well}, given $w\in\R^d$, if $R\|w\|\le2$, then
  \begin{align*}
    \cRlog(w)\ge \frac{R^2}{2U},
  \end{align*}
  while if $R\|w\|\ge2$, then
  \begin{align*}
    \cRlog(w)\ge \frac{R}{U\|w\|}.
  \end{align*}
\end{lemma}
\begin{proof}
  First, since $\llog(z)\ge\llog\del{|z|}$,
  \begin{align}\label{eq:norm_lb_tmp1}
    \cRlog(w)=\bbE_{(x,y)\sim P}\sbr{\llog\del{y \langle w,x\rangle}}\ge\bbE_{x\sim P_x}\sbr{\llog\del{\envert{\langle w,x\rangle}}}.
  \end{align}
  Let $v$ denote an arbitrary vector that is orthogonal to $w$, and let $p$
  denote the density of the projection of $P_x$ onto the space spanned by $w$
  and $v$.
  Without loss of generality, we can assume $w$ has polar angle $0$.
  Then \cref{eq:norm_lb_tmp1} becomes
  \begin{align*}
    \cRlog(w)\ge\int_0^\infty\int_0^{2\pi}\llog\del{r\|w\|\envert{\cos(\theta)}}p(r,\theta)r\dif\theta\dif r.
  \end{align*}
  \Cref{cond:well} and \Cref{fact:risk_ub_lb} then imply
  \begin{align*}
    \cRlog(w) & \ge \frac{1}{U}\int_0^R\int_0^{2\pi}\llog\del{r\|w\|\envert{\cos(\theta)}}r\dif\theta\dif r \\
     & \ge \frac{1}{U}\frac{2}{\|w\|}\int_0^R\del{1-e^{-r\|w\|}}\dif r \\
     & =\frac{2}{U}\frac{1}{\|w\|^2}\del{e^{-R\|w\|}-1+R\|w\|}.
  \end{align*}
  If $R\|w\|\le2$, then because $e^{-z}-1+z\ge \frac{z^2}{4}$ when
  $0\le z\le 2$, we have
  \begin{align*}
    \cRlog(w)\ge \frac{2}{U}\frac{1}{\|w\|^2}\frac{R^2\|w\|^2}{4}=\frac{R^2}{2U}.
  \end{align*}
  Otherwise if $R\|w\|\ge2$, then because $e^{-z}-1+z\ge \frac{z}{2}$ when
  $z\ge2$, we have
  \begin{align*}
    \cRlog(w)\ge \frac{2}{U}\frac{1}{\|w\|^2}\frac{R\|w\|}{2}=\frac{R}{U\|w\|}.
  \end{align*}
\end{proof}

\subsection{Omitted proofs from \Cref{sec:log_ub_main}}\label{app_sec:log_ub_main}

In this section, we prove \Cref{thm:exact-opt,fact:log_ub_opt} using
\Cref{fact:log_ub}.

First, we prove the following norm lower bound on $\|w^*\|$, which covers
\Cref{fact:w*_norm_lb_main} and also the sub-exponential case.
\begin{lemma}[\bf \Cref{fact:w*_norm_lb_main}, including the sub-exponential case]
\label{fact:w*_norm_lb}
  Under \Cref{cond:well}, if $\|x\|\le B$ almost surely
  and $\opt<\frac{R^4}{200U^3B}$, then
  $\|w^*\|=\Omega\del[2]{\frac{1}{\sqrt{\opt}}}$; if $P_x$ is
  $(\alpha_1,\alpha_2)$-sub-exponential and
  $\opt\cdot\ln(1/\opt)<\frac{R^4}{200(1+2\alpha_1)\alpha_2U^3}$, then
  $\|w^*\|=\Omega\del[2]{\frac{1}{\sqrt{\opt\cdot\ln(1/\opt)}}}$.
\end{lemma}
\begin{proof}
  Suppose $\|x\|\le B$ almost surely.
  Since $\opt<\frac{R^4}{200U^3B}$, \Cref{fact:log_ub_ref_sol} implies
  \begin{align*}
    \cRlog(w^*)\le\inf_{\rho>0}\cRlog(\rho\baru)\le\sqrt{50UB\cdot\opt}<\sqrt{50UB\cdot \frac{R^4}{200U^3B}}=\frac{R^2}{2U}.
  \end{align*}
  Therefore it follows from \Cref{fact:cR_lb} that $R\|w^*\|\ge2$, and
  \begin{align*}
    \frac{R}{U\|w^*\|}\le\cRlog(w^*)\le\inf_{\rho>0}\cRlog(\rho\baru)\le\sqrt{50UB\cdot\opt},
  \end{align*}
  which implies
  \begin{align*}
    \|w^*\|\ge \frac{R}{U\sqrt{50UB}}\cdot\frac{1}{\sqrt{\opt}}.
  \end{align*}

  Now suppose $P_x$ is $(\alpha_1,\alpha_2)$-sub-exponential.
  Since
  $\opt\cdot\ln\del{\frac{1}{\opt}}<\frac{R^4}{200(1+2\alpha_1)\alpha_2U^3}$,
  \Cref{fact:log_ub_ref_sol} implies
  \begin{align*}
    \inf_{\rho>0}\cRlog(\rho\baru)\le\sqrt{50(1+2\alpha_1)\alpha_2 U\cdot\opt\cdot\ln\del{\frac{1}{\opt}}}<\sqrt{50(1+2\alpha_1)\alpha_2U\cdot \frac{R^4}{200(1+2\alpha_1)\alpha_2U^3}}=\frac{R^2}{2U}.
  \end{align*}
  Therefore it follows from \Cref{fact:cR_lb} that $R\|w^*\|\ge2$, and
  \begin{align*}
    \frac{R}{U\|w^*\|}\le\cRlog(w^*)\le\inf_{\rho>0}\cRlog(\rho\baru)\le\sqrt{50(1+2\alpha_1)\alpha_2U\cdot\opt\cdot\ln\del{\frac{1}{\opt}}}
  \end{align*}
  which implies
  \begin{align*}
    \|w^*\|\ge \frac{R}{U\sqrt{50(1+2\alpha_1)\alpha_2U}}\frac{1}{\sqrt{\opt\cdot\ln(1/\opt)}}.
  \end{align*}
\end{proof}

Now we can prove \Cref{thm:exact-opt}.
\begin{proof}[Proof of \Cref{thm:exact-opt}]
  If $\|x\|\le B$ almost surely, \Cref{fact:log_ub} implies
  \begin{align*}
    \cR_{0-1}(w^*)= O\del[3]{\max\cbr{\opt,\frac{C_\kappa}{\|w^*\|^2}}}.
  \end{align*}
  If $\opt\ge \frac{R^4}{200U^3B}$, then \Cref{thm:exact-opt} holds vacuously;
  otherwise \Cref{fact:w*_norm_lb} ensures
  $\|w^*\|=\Omega\del{\frac{1}{\sqrt{\opt}}}$, and thus
  \begin{align*}
    \cR_{0-1}(w^*)= O\del{\max\cbr{\opt,C_\kappa\cdot\opt}}=O\del{(1+C_\kappa)\opt}.
  \end{align*}

  The proof of the sub-exponential case is similar.
\end{proof}

Next, we analyze project gradient descent.
First we restate \Cref{fact:pgd_wt_baru_main,fact:wt_norm_lb_main}, and also
handle sub-exponential distributions.
\begin{lemma}[\bf \Cref{fact:pgd_wt_baru_main}, including the sub-exponential case]
\label{fact:pgd_wt_baru}
  Let the target optimization error $\epsopt\in(0,1)$ and the failure
  probability $\delta\in(0,1/e)$ be given.
  If $\|x\|\le B$ almost surely, then with $\eta=4/B^2$, using
  $O\del{\frac{(B+1)^2\ln(1/\delta)}{\epsilon\epsopt^2}}$ samples and
  $O\del{\frac{B^2}{\epsilon\epsopt}}$ iterations, with probability $1-\delta$,
  projected gradient descent outputs $w_t$ satisfying
  \begin{align}\label{eq:pgd_excess}
    \cRlog(w_t)\le\min_{0\le\rho\le1/\sqrt{\epsilon}}\cRlog(\rho\baru)+\epsopt.
  \end{align}

  If $P_x$ is $(\alpha_1,\alpha_2)$-sub-exponential, then with
  $\eta=\widetilde{\Theta}(1/d)$, using
  $\widetilde{O}\del{\frac{d\ln(1/\delta)^3}{\epsilon\epsopt^2}}$ samples and
  $\widetilde{O}\del{\frac{d\ln(1/\delta)^2}{\epsilon\epsopt}}$ iterations, with
  probability $1-\delta$, projected gradient descent outputs $w_t$ satisfying
  \cref{eq:pgd_excess}.
\end{lemma}

\begin{lemma}[\bf \Cref{fact:wt_norm_lb_main}, including the sub-exponential case]
\label{fact:wt_norm_lb}
  Under \Cref{cond:well}, suppose
  \begin{align*}
    \epsilon<\min\cbr{\frac{R^4}{36U^2},\frac{R^4}{72^2U^4}}\quad\textup{and}\quad\epsopt\le\sqrt{\epsilon},
  \end{align*}
  and that \cref{eq:pgd_excess} holds.
  If $\|x\|\le B$ almost surely and $\opt<\frac{R^4}{500U^3B}$,
  then $\|w_t\|=\Omega\del[2]{\min\cbr{\frac{1}{\sqrt{\epsilon}},\frac{1}{\sqrt{\opt}}}}$.

  On the other hand, if $P_x$ is $(\alpha_1,\alpha_2)$-sub-exponential, and
  $\opt\cdot\ln(1/\opt)<\frac{R^4}{500U^3(1+2\alpha_1)\alpha_2}$, then it holds that
  $\|w_t\|=\Omega\del[2]{\min\cbr[2]{\frac{1}{\sqrt{\epsilon}},\frac{1}{\sqrt{\opt\cdot\ln(1/\opt)}}}}$.
\end{lemma}

Next we prove \Cref{fact:pgd_wt_baru,fact:wt_norm_lb}.
We first consider bounded distributions, and then handle sub-exponential
distributions.
For simplicity, in the rest of this subsection we will use $\cR$ and $\hR$ to
denote $\cRlog$ and $\hRlog$, respectively.

\paragraph{Bounded distributions.}

First, here are some standard optimization and generalization results for
projected gradient descent.

\begin{lemma}\label{fact:pgd_opt}
  If $\|x_i\|\le B$ for all $1\le i\le n$, then $\hR$ is $\frac{B^2}{4}$-smooth.
  Moreover, if $w_0:=0$ and $\eta\le \frac{4}{B^2}$, then for all $t\ge1$,
  \begin{align*}
    \hR(w_t)\le\min_{w\in\cB(1/\sqrt{\epsilon})}\hR(w)+\frac{1}{2\eta\epsilon t}.
  \end{align*}
\end{lemma}
\begin{proof}
  Note that $\llog$ is $\frac{1}{4}$-smooth.
  To show $\hR$ is $\frac{B^2}{4}$-smooth, note that given any $w,w'\in\R^d$,
  \begin{align*}
    \enVert{\nhR(w)-\nhR(w')} & =\enVert{\frac{1}{n}\sum_{i=1}^{n}\del{\llog'\del{y_i \langle w,x_i\rangle}-\llog'\del{y_i \langle w',x_i\rangle}}y_ix_i} \\
     & \le \frac{1}{n}\sum_{i=1}^{n}\envert{\llog'\del{y_i \langle w,x_i\rangle}-\llog'\del{y_i \langle w',x_i\rangle}}B \\
      & \le \frac{B}{4n}\sum_{i=1}^{n}\envert{y_i \langle w,x_i\rangle-y_i \langle w',x_i\rangle} \\
      & \le \frac{B}{4n}\sum_{i=1}^{n}\|w-w'\|B=\frac{B^2}{4}\|w-w'\|.
  \end{align*}

  The following analysis basically comes from the proof of
  \citep[Theorem 6.3]{bubeck}; we include it for completeness, and also handle
  the last iterate.
  Let $w^*:=\argmin_{w\in\cB(1/\sqrt{\epsilon})}\hR(w)$.
  Convexity gives
  \begin{align*}
    \hR(w_t)-\hR(w^*)\le\ip{\nhR(w_t)}{w_t-w^*}=\ip{\nhR(w_t)}{w_t-w_{t+1}}+\ip{\nhR(w_t)}{w_{t+1}-w^*}.
  \end{align*}
  Smoothness implies
  \begin{align*}
    \ip{\nhR(w_t)}{w_t-w_{t+1}} & \le\hR(w_t)-\hR(w_{t+1})+\frac{B^2/4}{2}\|w_t-w_{t+1}\|^2 \\
     & \le \hR(w_t)-\hR(w_{t+1})+\frac{1}{2\eta}\|w_t-w_{t+1}\|^2.
  \end{align*}
  On the other hand, the projection step ensures
  \begin{align*}
    \ip{\nhR(w_t)}{w_{t+1}-w^*} & \le \frac{1}{\eta}\ip{w_t-w_{t+1}}{w_{t+1}-w^*} \\
     & =\frac{1}{2\eta}\del{\|w_t-w^*\|^2-\|w_{t+1}-w^*\|^2-\|w_t-w_{t+1}\|^2}.
  \end{align*}
  Therefore
  \begin{align*}
    \hR(w_t)-\hR(w^*)\le\hR(w_t)-\hR(w_{t+1})+\frac{1}{2\eta}\del{\|w_t-w^*\|^2-\|w_{t+1}-w^*\|^2},
  \end{align*}
  which implies
  \begin{align}\label{eq:pgd_opt_tmp1}
    \hR(w_{t+1})-\hR(w^*)\le \frac{1}{2\eta}\del{\|w_t-w^*\|^2-\|w_{t+1}-w^*\|^2}.
  \end{align}

  Next we show that $\hR(w_{t+1})\le\hR(w_t)$.
  Smoothness implies
  \begin{align*}
    \hR(w_{t+1})-\hR(w_t) & \le\ip{\nhR(w_t)}{w_{t+1}-w_t}+\frac{B^2/4}{2}\|w_{t+1}-w_t\|^2 \\
     & \le-\frac{1}{\eta}\|w_t-w_{t+1}\|^2+\frac{B^2/4}{2}\|w_{t+1}-w_t\|^2 \\
     & \le-\frac{1}{\eta}\|w_t-w_{t+1}\|^2+\frac{1}{2\eta}\|w_{t+1}-w_t\|^2 \\
     & =-\frac{1}{2\eta}\|w_{t+1}-w_t\|^2,
  \end{align*}
  where we also use the property of the projection step on the second line.

  It now follow from \cref{eq:pgd_opt_tmp1} and $\hR(w_{t+1})\le\hR(w_t)$ that
  for $t\ge1$,
  \begin{align*}
    \hR(w_t)\le\hR(w^*)+\frac{\|w_0-w^*\|^2}{2\eta t}\le\hR(w^*)+\frac{1}{2\eta\epsilon t}.
  \end{align*}
\end{proof}

\begin{lemma}\label{fact:pgd_gen}
  If $\|x\|\le B$ almost surely, then with probability $1-\delta$, for all
  $w\in\cB\del{\frac{1}{\sqrt{\epsilon}}}$,
  \begin{align*}
    \envert{\cR(w)-\hR(w)}\le \frac{2B}{\sqrt{\epsilon n}}+3\del{\frac{B}{\sqrt{\epsilon}}+1}\sqrt{\frac{\ln(4/\delta)}{2n}}.
  \end{align*}
\end{lemma}
\begin{proof}
  Note that $\llog(z)\le|z|+1$, therefore
  \begin{align*}
    \llog\del{y \langle w,x\rangle}\le\|w\|\|x\|+1\le \frac{B}{\sqrt{\epsilon}}+1.
  \end{align*}
  Since $\llog$ is $1$-Lipschitz continuous,
  \citep[Theorem 26.5, Lemma 26.9, Lemma 26.10]{understand_ml} imply that with
  probability $1-\delta$, for all $w\in\cB\del{\frac{1}{\sqrt{\epsilon}}}$,
  \begin{align*}
    \cR(w)-\hR(w)\le\frac{2B}{\sqrt{\epsilon n}}+3\del{\frac{B}{\sqrt{\epsilon}}+1}\sqrt{\frac{\ln(2/\delta)}{2n}}.
  \end{align*}
  Next we can just apply the same technique and get a uniform deviation bound on
  $\hR(w)-\cR(w)$.
\end{proof}

We can now prove \Cref{fact:pgd_wt_baru}.
\begin{proof}[Proof of \Cref{fact:pgd_wt_baru} for bounded distributions]
  \Cref{fact:pgd_opt} implies that
  \begin{align*}
    \hR(w_t)-\min_{0\le\rho\le1/\sqrt{\epsilon}}\cR(\rho\baru)\le\frac{1}{2\eta\epsilon t}=\frac{B^2}{8\epsilon t}.
  \end{align*}
  Moreover, \Cref{fact:pgd_gen} ensures with probability $1-\delta$, for all
  $w\in\cB_2\del{\frac{1}{\sqrt{\epsilon}}}$,
  \begin{align*}
    \envert{\hR(w)-\cR(w)}\le \frac{2B}{\sqrt{\epsilon n}}+3\del{\frac{B}{\sqrt{\epsilon}}+1}\sqrt{\frac{\ln(4/\delta)}{2n}}=O\del{(B+1)\sqrt{\frac{\ln(1/\delta)}{\epsilon n}}}.
  \end{align*}
  Therefore, to ensure
  $\cR(w_t)-\min_{0\le\rho\le1/\sqrt{\epsilon}}\cR(\rho\baru)\le\epsopt$,
  we only need
  \begin{align*}
    O\del{\frac{B^2}{\epsilon\epsopt}}\textup{ steps},\quad\textup{and}\quad O\del{\frac{(B+1)^2\ln(1/\delta)}{\epsilon\epsopt^2}}\textup{ samples}.
  \end{align*}
\end{proof}

Next we prove the norm lower bound on $\|w_t\|$.
\begin{proof}[Proof of \Cref{fact:wt_norm_lb}]
  First, we consider the case $\|x\|\le B$ almost surely.
  It follows from \Cref{fact:log_ub_ref_sol} that
  \begin{align}\label{eq:rho_baru_risk}
    \cR(\rho\baru)\le \frac{12U}{\rho}+\rho B\cdot\opt.
  \end{align}
  Let $\bar{\rho}:=\sqrt{\frac{12U}{B\cdot\opt}}$.
  We consider two cases below, $\bar{\rho}\le \frac{1}{\sqrt{\epsilon}}$ or
  $\bar{\rho}\ge \frac{1}{\sqrt{\epsilon}}$.

  First, we assume $\bar{\rho}\le \frac{1}{\sqrt{\epsilon}}$.
  Then by the conditions of \Cref{fact:wt_norm_lb} and \cref{eq:rho_baru_risk},
  we have
  \begin{align*}
    \cR(w_t)\le\cR(\bar{\rho}\baru)+\epsopt& \le2\sqrt{12UB\cdot\opt}+\sqrt{\epsilon} \\
     & <2\sqrt{12UB\cdot \frac{R^4}{500U^3B}}+\sqrt{\frac{R^4}{36U^2}} \\
     & <2 \frac{R^2}{6U}+\frac{R^2}{6U}=\frac{R^2}{2U}.
  \end{align*}
  It then follows from \Cref{fact:cR_lb} that $R\|w_t\|\ge2$, and
  \begin{align*}
    \frac{R}{U\|w_t\|}\le\cR(w_t)\le\cR(\bar{\rho}\baru)+\epsopt\le2\sqrt{12UB\cdot\opt}+\sqrt{\epsilon}.
  \end{align*}
  since $\bar{\rho}\le\frac{1}{\sqrt{\epsilon}}$,
  \begin{align*}
    \sqrt{\epsilon}\le \frac{1}{\bar{\rho}}=\sqrt{\frac{B\cdot\opt}{12U}}.
  \end{align*}
  As a result, $\frac{R}{U\|w_t\|}=O\del[1]{\sqrt{\opt}}$, which implies
  $\|w_t\|=\Omega\del{\frac{1}{\sqrt{\opt}}}$.

  Next, assume $\bar{\rho}\ge \frac{1}{\sqrt{\epsilon}}$, which implies that
  \begin{align*}
    \frac{B\cdot\opt}{12U}\le\epsilon,\quad\textup{and}\quad B\cdot\opt\le12U\epsilon.
  \end{align*}
  Moreover, \cref{eq:rho_baru_risk} implies
  \begin{align*}
    \cR\del{\frac{1}{\sqrt{\epsilon}}\baru}\le12U\sqrt{\epsilon}+\frac{1}{\sqrt{\epsilon}}B\cdot\opt\le12U\sqrt{\epsilon}+\frac{1}{\sqrt{\epsilon}}12U\epsilon=24U\sqrt{\epsilon}.
  \end{align*}
  Then because
  \begin{align*}
    \cR(w_t)\le\cR\del{\frac{1}{\sqrt{\epsilon}}\baru}+\epsopt\le24U\sqrt{\epsilon}+\sqrt{\epsilon}<24U\sqrt{\frac{R^4}{72^2U^4}}+\sqrt{\frac{R^4}{36U^2}}=\frac{R^2}{2U},
  \end{align*}
  it further follows from \Cref{fact:cR_lb} that $R\|w_t\|\ge2$, and
  \begin{align*}
    \frac{R}{U\|w_t\|}\le\cR\del{\frac{1}{\sqrt{\epsilon}}\baru}+\epsopt\le24U\sqrt{\epsilon}+\sqrt{\epsilon},
  \end{align*}
  therefore $\|w_t\|=\Omega\del{\frac{1}{\sqrt{\epsilon}}}$.

  Now assume $P_x$ is $(\alpha_1,\alpha_2)$-sub-exponential.
  \Cref{fact:log_ub_ref_sol} implies
  \begin{align*}
    \cR(\rho\baru)\le\frac{12U}{\rho}+(1+2\alpha_1)\alpha_2\rho\cdot\opt\cdot\ln\del{\frac{1}{\opt}}.
  \end{align*}
  Let
  \begin{align*}
    \bar{\rho}:=\sqrt{\frac{12 U}{(1+2\alpha_1)\alpha_2\cdot\opt\cdot\ln(1/\opt)}},
  \end{align*}
  and similarly consider the two cases $\bar{\rho}\le \frac{1}{\sqrt{\epsilon}}$
  and $\bar{\rho}\ge \frac{1}{\sqrt{\epsilon}}$, we can finish the proof.
\end{proof}

Now we are ready to prove \Cref{fact:log_ub_opt}.
\begin{proof}[Proof of \Cref{fact:log_ub_opt} for bounded distributions.]
  First, note that if $\epsilon$ or $\opt$ does not satisfy the conditions of \Cref{fact:wt_norm_lb}, then \Cref{fact:log_ub_opt} holds vacuously.
  Under the conditions of \Cref{fact:pgd_wt_baru,fact:wt_norm_lb}, let
  $\epsopt:=\epsilon^{3/2}$, we have that projected gradient descent can find
  $w_t$ satisfying
  \begin{align*}
    \cRlog(w_t)\le\min_{0\le\rho\le1/\sqrt{\epsilon}}\cRlog(\rho\baru)+\epsilon^{3/2},
  \end{align*}
  and
  \begin{align*}
    \|w_t\|=\Omega\del[3]{\min\cbr{\frac{1}{\sqrt{\epsilon}},\frac{1}{\sqrt{\opt}}}}.
  \end{align*}

  Now we just need to invoke \Cref{fact:log_ub}.
  If $\epsilon\le\opt$, then $\|w_t\|=\Omega\del{\frac{1}{\sqrt{\opt}}}$, and
  \Cref{fact:log_ub} implies
  \begin{align*}
     & \cR_{0-1}(w_t) \\
    = & \ O\del[4]{\max\cbr{\opt,\sqrt{\epsilon^{3/2}\sqrt{\opt}},C_\kappa\cdot\opt}} \\
    = & \ O\del{(1+C_\kappa)\opt}.
  \end{align*}
  If $\epsilon\ge\opt$, then $\|w_t\|=\Omega\del{\frac{1}{\sqrt{\epsilon}}}$,
  and similarly we can show
  \begin{align*}
     & \cR_{0-1}(w_t) \\
    = & \ O\del[3]{\max\cbr{\opt,\sqrt{\epsilon^{3/2}\sqrt{\epsilon}},C_\kappa\epsilon}} \\
    = & \ O\del{(1+C_\kappa)(\opt+\epsilon)}.
  \end{align*}
  The sample and iteration complexity follow from \Cref{fact:pgd_wt_baru} and
  that $\epsopt=\epsilon^{3/2}$.
\end{proof}

\paragraph{Sub-exponential distributions.}

Next we handle $(\alpha_1,\alpha_2)$-sub-exponential distributions.
We will prove \Cref{fact:pgd_wt_baru} for sub-exponential distributions; the
rest of the proof is similar to the bounded case and thus omitted.

Let the target zero-one error $\epsilon$, the target optimization error
$\epsopt$, and failure probability $\delta$ be given.
Given $r>0$, we overload the notation a little bit and let
\begin{align*}
  \delta(r):=d\alpha_1\exp\del{-\frac{r}{\alpha_2\sqrt{d}}}.
\end{align*}
In particular, note that
\begin{align*}
  \pr_{x\sim P_x}\del{\|x\|\ge r}\le \sum_{j=1}^{d}\pr\del{|x_j|\ge \frac{r}{\sqrt{d}}}\le d\alpha_1\exp\del{-\frac{r}{\sqrt{d}\alpha_2}}=\delta(r).
\end{align*}

Let $B>1$ be large enough such that
\begin{align}\label{eq:radius_cond}
  \del{1-\delta(B)}^{100(B+1)^2\ln(4/\delta)/(\epsilon\epsopt^2)}\ge1-\delta,\quad\textup{and}\quad \alpha_1(\alpha_2+B)\exp\del{-\frac{B}{\alpha_2}}\le\epsopt\sqrt{\epsilon}.
\end{align}
We have the following bound on $B$.
\begin{lemma}\label{fact:radius_bound}
  To satisfy \cref{eq:radius_cond}, it is enough to let
  \begin{align*}
    B=\Omega\del{\sqrt{d}\ln\del{\frac{d}{\epsilon\epsopt\delta}}}.
  \end{align*}
\end{lemma}
\begin{proof}
  First, we let $B\ge\alpha_2\sqrt{d}\ln(2d\alpha_1)$ to ensure
  $\delta(B)\le1/2$.
  Since for $0\le z\le1/2$, we have $e^{-z}\ge1-z\ge e^{-2z}$, to satisfy the
  first condition of \cref{eq:radius_cond}, it is enough to ensure
  \begin{align*}
    e^{-\delta(B)\cdot200(B+1)^2\ln(4/\delta)/(\epsilon\epsopt^2)}\ge e^{-\delta},\quad\textup{equivalently}\quad\delta(B)\le \frac{\delta\epsilon\epsopt^2}{200(B+1)^2\ln(4/\delta)}.
  \end{align*}
  Invoking the definition of $\delta(B)$, we only need
  \begin{align*}
    B\ge\alpha_2\sqrt{d}\ln\del{\frac{200(B+1)^2d\alpha_1\ln(4/\delta)}{\delta\epsilon\epsopt^2}}.
  \end{align*}
  In other words, it is enough if
  $B=\Omega\del{\sqrt{d}\ln\del{\frac{d}{\epsilon\epsopt\delta}}}$.

  Similarly, to satisfy the second condition of \cref{eq:radius_cond}, we only
  need
  \begin{align*}
    B\ge\alpha_2\ln\del{\frac{\alpha_1(\alpha_2+B)}{\epsopt\sqrt{\epsilon}}},
  \end{align*}
  and it is enough if
  $B=\Omega\del{\sqrt{d}\ln\del{\frac{d}{\epsilon\epsopt\delta}}}$.
\end{proof}

Now we define a truncated logistic loss $\llog^\circ$ as following:
\begin{equation*}
  \llog^\circ(z):=
  \begin{dcases}
    \llog\del{-\frac{B}{\sqrt{\epsilon}}} & \textup{if }z\le-\frac{B}{\sqrt{\epsilon}}, \\
    \llog(z) & \textup{if }z\ge-\frac{B}{\sqrt{\epsilon}}.
  \end{dcases}
\end{equation*}
We also let $\cR^\circ(w)$ and $\hR^\circ(w)$ denote the population and
empirical risk with the truncated logistic loss.
We have the next result.
\begin{lemma}\label{fact:sub_exp_pop}
  Suppose $B>1$ is chosen according to \cref{eq:radius_cond}.
  Using a constant step size $4/B^2$, and
  \begin{align*}
    \frac{100(B+1)^2\ln(4/\delta)}{\epsilon\epsopt^2}\textup{ samples,}\quad\textup{and}\quad \frac{B^2}{4\epsilon\epsopt}\textup{ steps,}
  \end{align*}
  with probability $1-2\delta$, projected gradient descent can ensure
  \begin{align*}
    \cR^\circ(w_t)\le\min_{0\le\rho\le1/\sqrt{\epsilon}}\cR(\rho\baru)+\epsopt.
  \end{align*}
\end{lemma}
\begin{proof}
  It follows from \cref{eq:radius_cond} that with probability $1-\delta$, it
  holds that $\|x_i\|\le B$ for all training examples.
  Therefore \Cref{fact:pgd_opt} implies that
  \begin{align*}
    \hR(w_t)\le \min_{0\le\rho\le1/\sqrt{\epsilon}}\hR(\rho\baru)+\frac{B^2}{8\epsilon t}.
  \end{align*}
  Since $\|x_i\|\le B$, and the domain is
  $\cB(1/\sqrt{\epsilon})$, it follows that
  \begin{align*}
    \hR^\circ(w_t)\le \min_{0\le\rho\le1/\sqrt{\epsilon}}\hR^\circ(\rho\baru)+\frac{B^2}{8\epsilon t}.
  \end{align*}
  Letting $t=\frac{B^2}{4\epsilon\epsopt}$, we get
  \begin{align}\label{eq:sub_exp_opt}
    \hR^\circ(w_t)\le \min_{0\le\rho\le1/\sqrt{\epsilon}}\hR^\circ(\rho\baru)+\frac{\epsopt}{2}.
  \end{align}

  Note that by the construction of the truncated logistic loss, it holds that
  \begin{align*}
    \llog^\circ(z)\le \frac{B}{\sqrt{\epsilon}}+1.
  \end{align*}
  Then by invoking the standard Rademacher complexity results
  \citep[Theorem 26.5, Lemma 26.9, Lemma 26.10]{understand_ml}, and recall that
  we work under the event $\|x_i\|\le B$ for all training examples, we can show
  with probability $1-2\delta$ that for all $w\in\cB(1/\sqrt{\epsilon})$,
  \begin{align*}
    \envert{\cR^\circ(w)-\hR^\circ(w)} & \le \frac{2B}{\sqrt{\epsilon n}}+3\del{\frac{B}{\sqrt{\epsilon}}+1}\sqrt{\frac{\ln(4/\delta)}{2n}} \\
     & \le \frac{2(B+1)}{\sqrt{\epsilon}}\sqrt{\frac{\ln(4/\delta)}{n}}+\frac{3(B+1)}{\sqrt{\epsilon}}\sqrt{\frac{\ln(4/\delta)}{2n}} \\
     & \le5(B+1)\sqrt{\frac{\ln(4/\delta)}{\epsilon n}}.
  \end{align*}
  Letting $n=\frac{100(B+1)^2\ln(4/\delta)}{\epsilon\epsopt^2}$, we have
  \begin{align}\label{eq:sub_exp_gen}
    \envert{\cR^\circ(w)-\hR^\circ(w)}\le \frac{\epsopt}{2}.
  \end{align}

  It then follows from \cref{eq:sub_exp_opt,eq:sub_exp_gen} that with
  probability $1-2\delta$,
  \begin{align*}
    \cR^\circ(w_t)\le\min_{0\le\rho\le1/\sqrt{\epsilon}}\cR^\circ(\rho\baru)+\epsopt\le\min_{0\le\rho\le1/\sqrt{\epsilon}}\cR(\rho\baru)+\epsopt,
  \end{align*}
  where we use $\llog^\circ\le\llog$ in the last inequality.
\end{proof}

Finally, we show that $\cR^\circ(w_t)$ is close to $\cR(w_t)$.
\begin{lemma}
  For all $w\in\cB(1/\sqrt{\epsilon})$, it holds that
  $\cR^\circ(w)\ge\cR(w)-\epsopt$.
\end{lemma}
\begin{proof}
  Note that if
  $\llog\del{y \langle w,x\rangle}\ne\llog^\circ\del{y \langle w,x\rangle}$,
  then $y \langle w,x\rangle\le-B/\sqrt{\epsilon}$, which implies
  $\envert{\langle w,x\rangle}\ge B/\sqrt{\epsilon}$.
  Moreover, in this case
  \begin{align*}
    \llog\del{y \langle w,x\rangle}-\llog^\circ\del{y \langle w,x\rangle}\le\llog\del{y \langle w,x\rangle}-\llog(0)\le\envert{\langle w,x\rangle}.
  \end{align*}
  Therefore
  \begin{align*}
    \cR(w)-\cR^\circ(w) & =\bbE_{x\sim P_x}\sbr{\llog\del{y \langle w,x\rangle}-\llog^\circ\del{y \langle w,x\rangle}}\le\bbE_{x\sim P_x}\sbr{\envert{\langle w,x\rangle}\1_{\envert{\langle w,x\rangle}\ge B/\sqrt{\epsilon}}}.
  \end{align*}
  We can then invoke \cref{eq:log_diff_approx_err_tail} and get
  \begin{align}\label{eq:circ_approx_tmp}
    \cR(w)-\cR^\circ(w)\le\alpha_1\del{\alpha_2\|w\|+\frac{B}{\sqrt{\epsilon}}}\exp\del{-\frac{B}{\alpha_2\|w\|\sqrt{\epsilon}}}.
  \end{align}
  Note that the right hand side of \cref{eq:circ_approx_tmp} is increasing with
  $\|w\|$, therefore we can let $\|w\|$ be $1/\sqrt{\epsilon}$ and get
  \begin{align*}
    \cR(w)-\cR^\circ(w)\le\alpha_1 \frac{\alpha_2+B}{\sqrt{\epsilon}}\exp\del{-\frac{B}{\alpha_2}}\le\epsopt,
  \end{align*}
  where we use \cref{eq:radius_cond} in the last inequality.
\end{proof}

Now putting everything together, under the conditions of
\Cref{fact:sub_exp_pop}, with probability $1-2\delta$, projected gradient
descent ensures
$\cR(w_t)\le\min_{0\le\rho\le1/\sqrt{\epsilon}}\cR(\rho\baru)+2\epsopt$.
Moreover, by applying \Cref{fact:radius_bound} to \Cref{fact:sub_exp_pop}, we
can see the sample complexity is
$\widetilde{O}\del{d\ln(1/\delta)^3/(\epsilon\epsopt^2)}$, and the iteration
complexity is $\widetilde{O}\del{d\ln(1/\delta)^2/(\epsilon\epsopt)}$.

\subsection{Omitted proofs from \Cref{sec:log_ub_01}}\label{app_sec:log_ub_01}

In this section, we prove \Cref{fact:log_ub}.
We first prove the following approximation bound after we replace the true label with the
label given by the ground-truth solution, which covers
\Cref{fact:noisy_label_main} and sub-exponential distributions.
\begin{lemma}[\bf \Cref{fact:noisy_label_main}, including the sub-exponential case]
\label{fact:noisy_label}
 For $\ell\in\{\llog,\ell_h\}$, if $\|x\|\le B$ almost surely,
  \begin{align*}
    \envert{\textup{term}~\eqref{eq:log_diff_y_baru}}\le B\|\barw-\hw\|\cdot\opt.
  \end{align*}
  If $P_x$ is $(\alpha_1,\alpha_2)$-sub-exponential, then
  \begin{align*}
    \envert{\textup{term}~\eqref{eq:log_diff_y_baru}}\le(1+2\alpha_1)\alpha_2\|\barw-\hw\|\cdot\opt\cdot\ln(1/\opt).
  \end{align*}
\end{lemma}
\begin{proof}
  Note that for both the logistic loss and the hinge loss, it holds that
  $\ell(-z)-\ell(z)=z$, therefore
  \begin{align}\label{eq:log_diff_approx_err}
    \textup{term~\eqref{eq:log_diff_y_baru}}=\bbE_{(x,y)\sim P}\sbr{\1_{y\ne\sign\del{\langle\barw,x\rangle}}\cdot y \langle\barw-\hw,x\rangle},
  \end{align}
  It then follows from the triangle inequality that
  \begin{align*}
    \envert{\textup{term}~\eqref{eq:log_diff_y_baru}}\le\bbE_{(x,y)\sim P}\sbr{\1_{y\ne\sign\del{\langle\barw,x\rangle}}\envert{\langle\barw-\hw,x\rangle}}
  \end{align*}
  Now we can invoke \Cref{fact:opt_ip_bound} with $w=\barw$ and $w'=\barw-\hw$
  to prove \Cref{fact:noisy_label}.
\end{proof}

Next we prove the lower bound on term~\eqref{eq:log_diff_approx_equiv_1}.
\begin{proof}[Proof of \Cref{fact:ground_truth_diff}]
  Note that in term~\eqref{eq:log_diff_approx_equiv_1}, we only care about
  $\langle\hw,x\rangle$ and $\langle\barw,x\rangle$, therefore we can focus on
  the two-dimensional space spanned by $\barw$ and $\hw$.
  Let $\varphi$ denote the angle between $\barw$ and $\hw$.
  Without loss of generality, we can consider the following graph, where we put
  $\barw$ at angle $0$, and $\hw$ at angle $\varphi$.
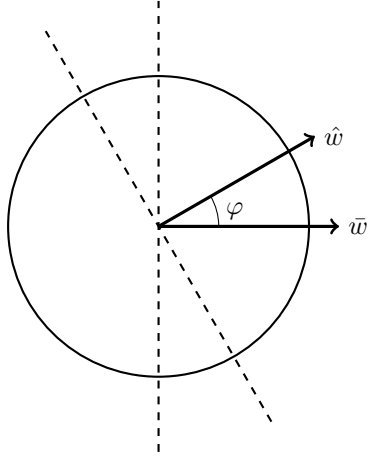
  \begin{center}
  \begin{tikzpicture}[scale=2]
    \draw [thick] (0,0) circle (1);
    \draw [thick,dashed] (0,-1.5) -- (0,1.5);
    \draw [->,very thick] (0,0) -- (1.2,0);
    \draw (1.2,0) node[anchor=west]{$\barw$};
    \draw [thick,dashed] (-60:1.5) -- (120:1.5);
    \draw [->,very thick] (0,0) -- (30:1.2);
    \draw (30:1.2) node[anchor=west]{$\hw$};
    \draw (0.4,0) arc (0:30:0.4);
    \draw (15:0.4) node[anchor=west]{$\varphi$};
  \end{tikzpicture}
  \end{center}
  We divide the graph into four parts given by different polar angles:
  (i) $(-\frac{\pi}{2},-\frac{\pi}{2}+\varphi)$,
  (ii) $(-\frac{\pi}{2}+\varphi,\frac{\pi}{2})$,
  (iii) $(\frac{\pi}{2},\frac{\pi}{2}+\varphi)$, and
  (iv) $(\frac{\pi}{2}+\varphi,\frac{3\pi}{2})$.
  Note that term~\eqref{eq:log_diff_approx_equiv_1} is $0$ on parts (ii) and (iv),
  therefore we only need to consider parts (i) and (iii):
  \begin{align*}
    \textup{term~\eqref{eq:log_diff_approx_equiv_1}} & =\bbE_{\textrm{(i) and (iii)}}\sbr{\ell\del{\sign\del{\langle\barw,x\rangle}\langle\hw,x\rangle}-\ell\del{\sign\del{\langle\hw,x\rangle}\langle\hw,x\rangle}} \\
     & =\bbE_{\textrm{(i) and (iii)}}\sbr{-\sign\del{\langle\barw,x\rangle}\langle\hw,x\rangle}.
  \end{align*}
  Here we use the fact that $\ell(-z)-\ell(z)=z$ for both the logistic loss and
  the hinge loss.

  For simplicity, let $p$ denote the density of the projection of $P_x$ onto the
  space spanned by $\hw$ and $\barw$.
  Under \Cref{cond:well}, we have
  \begin{align*}
      \textup{term~\eqref{eq:log_diff_approx_equiv_1}} & =\bbE_{\textrm{(i) and (iii)}}\sbr{-\sign\del{\langle\barw,x\rangle}\langle\hw,x\rangle} \\
       & =\int_0^\infty\int_{-\frac{\pi}{2}}^{-\frac{\pi}{2}+\varphi}-r\|\hw\|\cos(\varphi-\theta)p(r,\theta)r\dif\theta\dif r+\int_0^\infty\int_{\frac{\pi}{2}}^{\frac{\pi}{2}+\varphi}r\|\hw\|\cos(\theta-\varphi)p(r,\theta)r\dif\theta\dif r \\
       & \ge \frac{2}{U}\int_0^R\int_0^\varphi r\|\hw\|\sin(\theta)r\dif\theta\dif r \\
       & =\frac{2R^3\|\hw\|\del{1-\cos(\varphi)}}{3U}\ge \frac{4R^3\|\hw\|\varphi^2}{3U\pi^2},
  \end{align*}
  where we use the fact that $1-\cos(\varphi)\ge \frac{2\varphi^2}{\pi^2}$ for
  all $\varphi\in[0,\pi]$.
\end{proof}

Next, we prove the following upper bound on term~\eqref{eq:log_diff_approx_equiv_2}, covering
\Cref{fact:rotate_diff_main} and the sub-exponential case.
\begin{lemma}[\bf \Cref{fact:rotate_diff_main}, including the sub-exponential case]
\label{fact:rotate_diff}
  For $\ell=\ell_h$, term~\eqref{eq:log_diff_approx_equiv_2} is $0$.
  For $\ell=\llog$, under \Cref{cond:rad}, if $\|x\|\le B$ almost surely,
  then
  \begin{align*}
    \envert{\textup{term~\eqref{eq:log_diff_approx_equiv_2}}}\le12C_\kappa\cdot \frac{\varphi(\hw,\barw)}{\|\hw\|},
  \end{align*}
  where $C_\kappa:=\int_0^B\kappa(r)\dif r$,
  while if $P_x$ is $(\alpha_1,\alpha_2)$-sub-exponential, then
  \begin{align*}
    \envert{\textup{term~\eqref{eq:log_diff_approx_equiv_2}}}\le2\alpha_1\opt^2+12C_\kappa\cdot \frac{\varphi(\hw,\barw)}{\|\hw\|},
  \end{align*}
  where $C_\kappa:=\int_0^{3\alpha_2\ln(1/\opt)}\kappa(r)\dif r$.
\end{lemma}
\begin{proof}
  For the hinge loss, term~\eqref{eq:log_diff_approx_equiv_2} is $0$ simply
  because $\ell_h(z)=0$ when $z\ge0$.
  Next we consider the logistic loss.

  Note that term~\eqref{eq:log_diff_approx_equiv_2} only depends on
  $\langle\hw,x\rangle$ and $\langle\barw,x\rangle$, therefore we can focus on
  the subspace spanned by $\hw$ and $\barw$.
  For simplicity, let $p$ denote the density function of the projection of $P_x$
  onto the space spanned by $\hw$ and $\barw$.
  Moreover, without loss of generality we can assume $\barw$ has polar angle $0$
  while $\hw$ has polar angle $\varphi$, where we let $\varphi$ denote
  $\varphi(\hw,\barw)$ for simplicity.
  It then follows that
  \begin{align*}
    \textup{term~\eqref{eq:log_diff_approx_equiv_2}} & =\int_0^\infty\int_0^{2\pi}\llog\del{r\|\hw\|\envert{\cos(\theta-\varphi)}}p(r,\theta)r\dif\theta\dif r-\int_0^\infty\int_0^{2\pi}\llog\del{r\|\hw\|\envert{\cos(\theta)}}p(r,\theta)r\dif\theta\dif r\\
    & =\int_0^\infty\int_0^{2\pi}\llog\del{r\|\hw\|\envert{\cos(\theta)}}\del{p(r,\theta+\varphi)-p(r,\theta)}r\dif\theta\dif r.
  \end{align*}

  First, if $\|x\|\le B$ almost surely, then
  \begin{align*}
    \envert{\textup{term~\eqref{eq:log_diff_approx_equiv_2}}} & \le\int_0^B\int_0^{2\pi}\llog\del{r\|\hw\|\envert{\cos(\theta)}}\envert{p(r,\theta+\varphi)-p(r,\theta)}r\dif\theta\dif r \\
     & \le\int_0^B\int_0^{2\pi}\llog\del{r\|\hw\|\envert{\cos(\theta)}}\cdot\kappa(r)\varphi\cdot r\dif\theta\dif r \\
     & =\varphi\int_0^B\kappa(r)\del{\int_0^{2\pi}\llog\del{r\|\hw\|\envert{\cos(\theta)}}r\dif\theta}\dif r.
  \end{align*}
  Then \Cref{fact:risk_ub_lb} implies
  \begin{align*}
    \envert{\textup{term~\eqref{eq:log_diff_approx_equiv_2}}}\le\varphi\int_0^B\kappa(r)\frac{8\sqrt{2}}{\|\hw\|}\dif r=8\sqrt{2}C_\kappa\cdot \frac{\varphi}{\|\hw\|}\le 12C_\kappa\cdot \frac{\varphi}{\|\hw\|}.
  \end{align*}

  Next, assume $P_x$ is $(\alpha_1,\alpha_2)$-sub-exponential.
  For a $2$-dimensional random vector $x$ sampled according to $p_V$, note that
  \begin{align*}
    \pr\del{\|x\|\ge B}\le\pr\del{|x_1|\ge \frac{\sqrt{2}B}{2}}+\pr\del{|x_2|\ge \frac{\sqrt{2}B}{2}}\le2\alpha_1\exp\del{-\frac{\sqrt{2}B}{2\alpha_2}}.
  \end{align*}
  Letting $B:=2\sqrt{2}\alpha_2\ln\del{\frac{1}{\opt}}$, we get
  $\pr\del{\|x\|\ge B}\le2\alpha_1\opt^2$.
  Since $\llog(z)\le1$ when $z\ge0$, we have
  \begin{align*}
    \textup{term~\eqref{eq:log_diff_approx_equiv_2}}\le & \ 2\alpha_1\opt^2 \\
     & \ +\int_0^B\int_0^{2\pi}\llog\del{r\|\hw\|\envert{\cos(\theta-\varphi)}}p(r,\theta)r\dif\theta\dif r-\int_0^B\int_0^{2\pi}\llog\del{r\|\hw\|\envert{\cos(\theta)}}p(r,\theta)r\dif\theta\dif r.
  \end{align*}
  Invoking the previous bound for bounded distributions, we get
  \begin{align*}
    \textrm{term~\eqref{eq:log_diff_approx_equiv_2}}\le2\alpha_1\opt^2+12\cdot \frac{\varphi}{\|\hw\|}\cdot\int_0^{2\sqrt{2}\alpha_2\ln\del{\frac{1}{\opt}}}\kappa(r)\dif r\le2\alpha_1\opt^2+12C_\kappa\cdot \frac{\varphi}{\|\hw\|},
  \end{align*}
  where $C_\kappa:=\int_0^{3\alpha_2\ln\del{\frac{1}{\opt}}}\kappa(r)\dif r$.
  Similarly, we can show
  \begin{align*}
    -\textup{term~\eqref{eq:log_diff_approx_equiv_2}}\le2\alpha_1\opt^2+12C_\kappa\cdot \frac{\varphi}{\|\hw\|}.
  \end{align*}
\end{proof}

Next we prove \Cref{fact:angle_01}, which is basically
\citep[Claim 3.4]{diakonikolas_adv_noise}.
\begin{proof}[Proof of \Cref{fact:angle_01}]
  Under \Cref{cond:well}, we have
  \begin{align*}
    \pr\del{\sign\del{\langle\hw,x\rangle}\ne\sign\del{\langle\barw,x\rangle}}\le2\varphi(\hw,\barw)\int_0^{\infty}\sigma(r)r\dif r\le 2U\varphi(\hw,\barw).
  \end{align*}
\end{proof}

Lastly, we prove \Cref{fact:log_ub} for sub-exponential distributions.
\begin{proof}[Proof of \Cref{fact:log_ub}, sub-exponential distributions]
  For simplicity, let $\varphi$ denotes $\varphi(\hw,\barw)$.
  \Cref{fact:noisy_label,fact:ground_truth_diff,fact:rotate_diff} imply
  \begin{align*}
    C_1\|\hw\|\varphi^2 & \le\epsopt+C_2\|\barw-\hw\|\cdot\opt\cdot\ln\del{\frac{1}{\opt}}+C_3\opt^2+C_4C_\kappa\cdot \frac{\varphi}{\|\hw\|} \\
     & \le\epsopt+C_2\|\hw\|\varphi\cdot\opt\cdot\ln\del{\frac{1}{\opt}}+C_3\opt^2+C_4C_\kappa\cdot \frac{\varphi}{\|\hw\|},
  \end{align*}
  where $C_1=\frac{4R^3}{3U\pi^2}$, and $C_2=(1+2\alpha_1)\alpha_2$, and
  $C_3=2\alpha_1$, and $C_4=12$.
  It follows that at least one of the following four cases is true:
  \begin{enumerate}
    \item $C_1\|\hw\|\varphi^2\le4\epsopt$, which implies
    $\varphi=O\del[1]{\sqrt{\epsopt/\|\hw\|}}$.

    \item $C_1\|\hw\|\varphi^2\le4C_2\|\hw\|\varphi\cdot\opt\cdot\ln\del{\frac{1}{\opt}}$,
    which implies $\varphi=O\del{\opt\ln\del{\frac{1}{\opt}}}$.

    \item $C_1\|\hw\|\varphi^2\le4C_3\opt^2$, which implies
    $\varphi=O(\opt)$ since $\|\hw\|=\Omega(1)$.

    \item Lastly,
    \begin{align}\label{eq:angle_ub_hw_subexp}
      C_1\|\hw\|\varphi^2\le4C_2C_\kappa\cdot \frac{\varphi}{\|\hw\|},\quad\textup{which implies}\quad \varphi=O\del{\frac{C_\kappa}{\|\hw\|^2}}.
    \end{align}
  \end{enumerate}

  Finally, we just need to invoke \Cref{fact:angle_01} to finish the proof.
\end{proof}

\subsection{Omitted proofs from \Cref{sec:log_ub_add}}\label{app_sec:log_ub_add}

We first prove the upper bound of term~\eqref{eq:log_diff_approx_equiv_2} under
\Cref{cond:well}, without assuming the radially Lipschitz condition.
\begin{proof}[Proof of \Cref{fact:rotate_diff_well}]
  Note that
  \begin{align*}
    \textup{term~\eqref{eq:log_diff_approx_equiv_2}}\le\bbE\sbr{\llog\del{\sign\del{\langle\hw,x\rangle}\langle\hw,x\rangle}}=\bbE\sbr{\llog\del{\envert{\langle \hw,x\rangle}}}\le \frac{12U}{\|\hw\|},
  \end{align*}
  where we invoke \Cref{fact:Rlog_inv_norm} at the end.
  Similarly, we can show
  \begin{align*}
    -\textup{term~\eqref{eq:log_diff_approx_equiv_2}}\le \frac{12U}{\|\barw\|}=\frac{12U}{\|\hw\|}
  \end{align*}
\end{proof}

Next we prove a general result similar to \Cref{fact:log_ub}.
\begin{theorem}\label{fact:log_ub_well}
  Under \Cref{cond:well}, suppose $\hw$ satisfies
  $\cRlog(\hw)\le\cRlog(\|\hw\|\baru)+\epsopt$ for some $\epsopt\in[0,1)$.
  If $\|x\|\le B$ almost surely, then
  \begin{align*}
    \varphi(\hw,\baru)=O\del[4]{\max\cbr{\opt,\sqrt{\frac{\epsopt}{\|\hw\|}},\frac{1}{\|\hw\|}}}.
  \end{align*}
  If $P_x$ is $(\alpha_1,\alpha_2)$-sub-exponential and $\|\hw\|=\Omega(1)$,
  then
  \begin{align*}
    \varphi(\hw,\baru)=O\del[4]{\max\cbr{\opt\cdot\ln\del{\frac{1}{\opt}},\sqrt{\frac{\epsopt}{\|\hw\|}},\frac{1}{\|\hw\|}}}.
  \end{align*}
\end{theorem}
\begin{proof}
  For simplicity, let $\varphi$ denote $\varphi(\hw,\baru)$.
  Consider the case $\|x\|\le B$ almost surely.
  The condition $\cRlog(\hw)\le\cRlog(\|\hw\|\baru)+\epsopt$, and
  \Cref{fact:noisy_label,fact:ground_truth_diff,fact:rotate_diff_well} imply
  \begin{align*}
    C_1\|\hw\|\varphi^2 & \le\epsopt+B\|\barw-\hw\|\cdot\opt+\frac{C_2}{\|\hw\|} \\
     & \le\epsopt+B\|\hw\|\varphi\cdot\opt+\frac{C_2}{\|\hw\|},
  \end{align*}
  where $C_1=4R^3/(3U\pi^2)$ and $C_2=12U$.
  Now at least one of the following three cases is true:
  \begin{enumerate}
    \item $C_1\|\hw\|\varphi^2\le3\epsopt$, which implies
    $\varphi=O\del[1]{\sqrt{\epsopt/\|\hw\|}}$;

    \item $C_1\|\hw\|\varphi^2\le3B\|\hw\|\varphi\cdot\opt$, which implies
    $\varphi=O(\opt)$;

    \item $C_1\|\hw\|\varphi^2\le3C_1/\|\hw\|$, which implies
    $\varphi=O(1/\|\hw\|)$.
  \end{enumerate}

  The proof of the sub-exponential case is similar.
\end{proof}

Now we prove \Cref{fact:log_ub_opt_well}.
\begin{proof}[Proof of \Cref{fact:log_ub_opt_well}]
  First, if $\epsilon$ or $\opt$ does not satisfy the conditions of \Cref{fact:wt_norm_lb}, then \Cref{fact:log_ub_opt_well} holds vacuously; therefore in the following we consider the settings of \Cref{fact:pgd_wt_baru,fact:wt_norm_lb} with
  $\epsopt=\sqrt{\epsilon}$.

  First, if $\|x\|\le B$ almost surely, \cref{eq:pgd_excess} and
  \Cref{fact:log_ub_well} imply
  \begin{align*}
    \varphi(w_t,\baru)=O\del[4]{\max\cbr{\opt,\sqrt{\frac{\epsopt}{\|w_t\|}},\frac{1}{\|w_t\|}}},
  \end{align*}
  and moreover \Cref{fact:wt_norm_lb} implies
  \begin{align*}
    \|w_t\|=\Omega\del[3]{\min\cbr{\frac{1}{\sqrt{\epsilon}},\frac{1}{\sqrt{\opt}}}}.
  \end{align*}
  If $\epsilon\le\opt$, then $\|w_t\|=\Omega\del{\frac{1}{\sqrt{\opt}}}$, and
  \begin{align*}
    \varphi(w_t,\baru) & =O\del[3]{\max\cbr{\opt,\sqrt{\epsopt\sqrt{\opt}},\sqrt{\opt}}} \\
     & =O\del[3]{\max\cbr{\opt,\sqrt{\sqrt{\epsilon}\sqrt{\opt}},\sqrt{\opt}}} \\
     & =O\del[3]{\max\cbr{\opt,\sqrt{\sqrt{\opt}\sqrt{\opt}},\sqrt{\opt}}}=O\del[1]{\sqrt{\opt}}.
  \end{align*}
  If $\epsilon\ge\opt$, then $\|w_t\|=\Omega\del{\frac{1}{\sqrt{\epsilon}}}$,
  and
  \begin{align*}
    \varphi(w_t,\baru) & =O\del[3]{\max\cbr{\opt,\sqrt{\epsopt\sqrt{\epsilon}},\sqrt{\epsilon}}} \\
     & =O\del[3]{\max\cbr{\opt,\sqrt{\sqrt{\epsilon}\sqrt{\epsilon}},\sqrt{\epsilon}}} \\
     & =O\del[1]{\sqrt{\opt+\epsilon}}.
  \end{align*}

  The proof for the sub-exponential case is similar.
\end{proof}

\section{Omitted proofs from \Cref{sec:hinge}}\label{app_sec:hinge}

In this section, we prove \Cref{fact:hinge_sgd}.
We first prove a bound on $\cR_h(\baru)$.
\begin{lemma}\label{fact:Rh_bound}
  If $\|x\|\le B$ almost surely, then $\cR_h(\baru)\le B\cdot\opt$, while if
  $P_x$ is $(\alpha_1,\alpha_2)$-sub-exponential, then
  $\cR_h(\baru)\le(1+2\alpha_1)\alpha_2\cdot\opt\cdot\ln(1/\opt)$.
\end{lemma}
\begin{proof}
  Note that
  \begin{align*}
    \cR_h(\baru)=\bbE_{(x,y)\sim P}\sbr{\ell_h\del{y \langle\baru,x\rangle}}=\bbE_{(x,y)\sim P}\sbr{\1_{\sign\del{\langle\baru,x\rangle\ne y}}\envert{\langle\baru,x\rangle}}.
  \end{align*}
  It then follows from \Cref{fact:opt_ip_bound} that if $\|x\|\le B$ almost
  surely, then
  \begin{align*}
    \cR_h(\baru)\le B\cdot\opt,
  \end{align*}
  while if $P_x$ is $(\alpha_1,\alpha_2)$-sub-exponential, then
  \begin{align*}
    \cR_h(\baru)\le(1+2\alpha_1)\alpha_2\cdot\opt\cdot\ln\del{\frac{1}{\opt}}.
  \end{align*}
\end{proof}

Next we prove the following result, which covers \Cref{fact:hinge_ub_main} but also
handles sub-exponential distributions.
\begin{lemma}[\bf \Cref{fact:hinge_ub_main}, including the sub-exponential case]
\label{fact:hinge_ub}
  Suppose \Cref{cond:well} holds.
  Consider an arbitrary $w\in\cD$, and let $\varphi$ denote $\varphi(w,\baru)$.
  If $\|x\|\le B$ almost surely, then
  \begin{align*}
    \cR_h(\barr\baru)\le\cR_h(\|w\|\baru)+O\del[1]{(\opt+\epsilon)^2}
  \end{align*}
  and
  \begin{align*}
    \cR_h(w)-\cR_h(\|w\|\baru)\ge \frac{4R^3}{3U\pi^2}\|w\|\varphi^2-B\|w\|\varphi\cdot\opt.
  \end{align*}

  If $P_x$ is $(\alpha_1,\alpha_2)$-sub-exponential, then
  \begin{align*}
    \cR_h(\barr\baru)\le\cR_h(\|w\|\baru)+O\del[1]{\del{\opt\cdot\ln(1/\opt)+\epsilon}^2}
  \end{align*}
  and
  \begin{align*}
     & \cR_h(w)-\cR_h(\|w\|\baru)\ge\frac{4R^3}{3U\pi^2}\|w\|\varphi^2 \\
     & \quad\quad\ -(1+2\alpha_1)\alpha_2\|w\|\varphi\cdot\opt\cdot\ln(1/\opt).
  \end{align*}
\end{lemma}
\begin{proof}
  First assume $\|x\|\le B$ almost surely.
  Note that $\ell_h$ is positive homogeneous, and thus for any positive constant
  $c$, we have $\cR_h(cw)=c\cR_h(w)$.
  Therefore, if $\barr\le\|w\|$, then
  \begin{align*}
    \cR_h(\barr\baru)=\frac{\barr}{\|w\|}\cR_h(\|w\|\baru)\le\cR_h(\|w\|\baru).
  \end{align*}
  If $\barr\ge\|w\|$, then
  \begin{align*}
    \cR_h(\barr\baru)=\cR_h\del{\|w\|\baru}+\cR_h(\baru)\del{\barr-\|w\|}\le\cR_h\del{\|w\|\baru}+\cR_h(\baru)\del{\barr-1},
  \end{align*}
  since $\|w\|\ge1$ for all $w\in\cD$.
  Recall that
  \begin{align*}
    \barr:=\frac{1}{\langle v,\baru\rangle}=\frac{1}{\cos\del{\varphi(v,\baru)}}\le \frac{1}{1-\varphi(v,\baru)^2/2},
  \end{align*}
  and therefore the first-phase of algorithm ensures $\barr=1+O(\opt+\epsilon)$ for bounded distributions, and $\barr=1+O\del{\opt\cdot\ln(1/\opt)+\epsilon}$ for sub-exponential distributions.
  It then follows that for bounded distributions,
  \begin{align*}
    \cR_h(\barr\baru) & \le\cR_h\del{\|w_t\|\baru}+\cR_h(\baru)\cdot O(\opt+\epsilon) \\
     & \le\cR_h\del{\|w_t\|\baru}+B\cdot\opt\cdot O(\opt+\epsilon) \\
     & =\cR_h\del{\|w_t\|\baru}+O\del[1]{(\opt+\epsilon)^2},
  \end{align*}
  where we apply \Cref{fact:Rh_bound} at the end.
  It also follows directly from
  \Cref{fact:noisy_label,fact:ground_truth_diff,fact:rotate_diff} that
  \begin{align*}
    \cR_h(w)-\cR_h(\|w\|\baru) & \ge \frac{4R^3}{3U\pi^2}\|w\|\varphi^2-B\enVert{w-\|w\|\baru}\cdot\opt \\
     & \ge \frac{4R^3}{3U\pi^2}\|w\|\varphi^2-B\|w\|\varphi\cdot\opt.
  \end{align*}

  The proof for the sub-exponential case is similar.
\end{proof}

Next we prove \Cref{fact:hinge_sgd}.
We first consider the bounded case.
\begin{proof}[Proof of \Cref{fact:hinge_sgd}, bounded distribution]
  Here we assume $\|x\|\le B$ almost surely.
  We will show that under the conditions of \Cref{fact:hinge_sgd}, then
  \begin{align}\label{eq:sgd_target_bounded}
    \bbE\sbr{\min_{0\le t<T}\varphi_t}=O(\opt+\epsilon),\quad\textup{where}\quad\varphi_t:=\varphi(w_t,\baru).
  \end{align}
  Further invoking \Cref{fact:angle_01} finishes the proof.

  Recall that at step $t$, after taking the expectation with respect to
  $(x_t,y_t)$, we have
  \begin{align}
    \bbE\sbr{\|w_{t+1}-\barr\baru\|^2} & \le\|w_t-\barr\baru\|^2-2\eta\ip{\nR_h(w_t)}{w_t-\barr\baru}+\eta^2B^2\cM(w_t) \nonumber \\
     & \le\|w_t-\barr\baru\|^2-2\eta\del{\cR_h(w_t)-\cR_h(\barr\baru)}+\eta^2B^2\cM(w_t). \label{eq:sgd_tmp}
  \end{align}

  First, \Cref{fact:hinge_ub} implies
  \begin{align*}
    \cR_h(w_t)-\cR_h(\barr\baru) & \ge\cR_h(w_t)-\cR_h(\|w_t\|\baru)-O\del[1]{(\opt+\epsilon)^2} \\
     & \ge2C_1\|w_t\|\varphi_t^2-B\|w_t\|\varphi_t\cdot\opt-O\del[1]{(\opt+\epsilon)^2},
  \end{align*}
  where $C_1:=2R^3/(3U\pi^2)$.
  Note that if $\varphi_t\le B\cdot\opt/C_1$, then \cref{eq:sgd_target_bounded}
  holds; therefore in the following we assume
  \begin{align}\label{eq:sgd_varphi_lb}
    \varphi_t\ge \frac{B}{C_1}\cdot\opt,
  \end{align}
  which implies
  \begin{align}\label{eq:sgd_tmp2}
    \cR_h(w_t)-\cR_h(\barr\baru)\ge C_1\|w_t\|\varphi_t^2-O\del[1]{(\opt+\epsilon)^2}\ge C_1\varphi_t^2-O\del[1]{(\opt+\epsilon)^2},
  \end{align}
  since $\|w\|\ge1$ for all $w\in\cD$.

  On the other hand, \cref{eq:sgd_varphi_lb} and \Cref{fact:angle_01} imply
  \begin{align*}
    \cM(w_t)=\cR_{0-1}(w_t)\le\opt+2U\varphi_t\le\del{\frac{C_1}{B}+2U}\varphi_t.
  \end{align*}
  Let
  \begin{align*}
    C_2:=\frac{C_1}{\del{\frac{C_1}{B}+2U}B^2}.
  \end{align*}
  Note that if $\varphi_t\le\epsilon$, then \cref{eq:sgd_target_bounded} is
  true; otherwise we can assume $\epsilon\le\varphi_t$, and let
  $\eta=C_2\epsilon$, we have
  \begin{align}\label{eq:sgd_tmp3}
    \eta B^2\cM(w_t)\le C_2\epsilon B^2\del{\frac{C_1}{B}+2U}\varphi_t=C_1\epsilon\varphi_t\le C_1\varphi_t^2.
  \end{align}

  Now \cref{eq:sgd_tmp,eq:sgd_tmp2,eq:sgd_tmp3} imply
  \begin{align*}
    \bbE\sbr{\|w_{t+1}-\barr\baru\|^2} & \le\|w_t-\barr\baru\|^2-2\eta C_1\varphi_t^2+\eta C_1\varphi_t^2+\eta\cdot O\del[1]{(\opt+\epsilon)^2} \\
     & =\|w_t-\barr\baru\|^2-\eta C_1\varphi_t^2+\eta\cdot O\del[1]{(\opt+\epsilon)^2}.
  \end{align*}
  Taking the expectation and average, we have
  \begin{align*}
    \bbE\sbr{\frac{1}{T}\sum_{t<T}^{}\varphi_t^2}\le \frac{\|w_0-\barr\baru\|^2}{\eta C_1T}+\frac{O\del[1]{(\opt+\epsilon)^2}}{C_1}.
  \end{align*}
  Note that
  \begin{align*}
    \|w_0-\barr\baru\|=\tan(\varphi_0)=O\del[1]{\sqrt{\opt+\epsilon}},
  \end{align*}
  and also recall $\eta=C_2\epsilon$, we have
  \begin{align*}
    \bbE\sbr{\frac{1}{T}\sum_{t<T}^{}\varphi_t^2}\le \frac{O(\opt+\epsilon)}{C_1C_2\epsilon T}+\frac{O\del[1]{(\opt+\epsilon)^2}}{C_1}.
  \end{align*}
  Letting $T=\Omega(1/\epsilon^2)$, we have
  \begin{align*}
    \bbE\sbr{\frac{1}{T}\sum_{t<T}^{}\varphi_t^2}\le O\del{(\opt+\epsilon)\epsilon}+O\del[1]{(\opt+\epsilon)^2}=O\del[1]{(\opt+\epsilon)^2},
  \end{align*}
  and thus \cref{eq:sgd_target_bounded} holds.
\end{proof}

Next we consider sub-exponential distributions.
We first prove the following bound on the square of norm.
\begin{lemma}\label{fact:sq_norm_sub_exp}
  Suppose $P_x$ is $(\alpha_1,\alpha_2)$-sub-exponential.
  Given any threshold $\tau>0$, it holds that
  \begin{align*}
    \bbE\sbr{\|x\|^2\1_{\|x\|\ge\tau}}\le d\alpha_1\del{\tau^2+2\sqrt{d}\alpha_2\tau+2d\alpha_2^2}\exp\del{-\frac{\tau}{\sqrt{d}\alpha_2}}.
  \end{align*}
\end{lemma}
\begin{proof}
  First recall that
  \begin{align*}
    \pr\del{\|x\|\ge\tau}\le \sum_{j=1}^{d}\pr\del{|x_j|\ge \frac{\tau}{\sqrt{d}}}\le d\alpha_1\exp\del{-\frac{\tau}{\sqrt{d}\alpha_2}}=:\delta(\tau).
  \end{align*}
  Let $\mu(\tau):=\pr\del{\|x\|\ge\tau}$.
  Integration by parts gives
  \begin{align*}
    \bbE\sbr{\|x\|^2\1_{\|x\|\ge\tau}}=\int_\tau^\infty r^2\cdot(-\dif\mu(r))=\tau^2\mu(\tau)+\int_\tau^\infty2r\mu(r)\dif r\le\tau^2\delta(\tau)+\int_\tau^\infty2r\delta(r)\dif r.
  \end{align*}
  Calculation gives
  \begin{align*}
    \bbE\sbr{\|x\|^2\1_{\|x\|\ge\tau}}\le d\alpha_1\del{\tau^2+2\sqrt{d}\alpha_2\tau+2d\alpha_2^2}\exp\del{-\frac{\tau}{\sqrt{d}\alpha_2}}.
  \end{align*}
\end{proof}

Now we are ready to prove \Cref{fact:hinge_sgd} for sub-exponential
distributions.
\begin{proof}[Proof of \Cref{fact:hinge_sgd}, sub-exponential distributions]
  At step $t$, we have
  \begin{align}
    \|w_{t+1}-\barr\baru\|^2 & \le\|w_t-\barr\baru\|^2-2\eta\ip{\ell'_h\del{y_t \langle w_t,x_t\rangle}y_tx_t}{w_t-\barr\baru}+\eta^2\ell'_h\del{y_t \langle w_t,x_t\rangle}^2\|x_t\|^2 \nonumber \\
     & =\|w_t-\barr\baru\|^2-2\eta\ip{\ell'_h\del{y_t \langle w_t,x_t\rangle}y_tx_t}{w_t-\barr\baru}-\eta^2\ell'_h\del{y_t \langle w_t,x_t\rangle}\|x_t\|^2, \label{eq:sgd_sub_exp_tmp1}
  \end{align}
  where we use $(\ell'_h)^2=-\ell'_h$.
  Next we bound
  $\bbE_{(x_t,y_t)}\sbr[1]{-\ell'_h\del{y_t \langle w_t,x_t\rangle}\|x_t\|^2}$.
  Let $\tau:=\sqrt{d}\alpha_2\ln(d/\epsilon)$.
  When $\|x_t\|\le\tau$, we have
  \begin{align*}
    \bbE\sbr{-\ell'_h\del{y_t \langle w_t,x_t\rangle}\|x_t\|^2\1_{\|x_t\|\le\tau}}\le\tau^2\cM(w_t)\le d\alpha_2^2\cM(w_t)\cdot\ln(d/\epsilon)^2.
  \end{align*}
  On the other hand, when $\|x_t\|\ge\tau$, \Cref{fact:sq_norm_sub_exp} implies
  \begin{align*}
    \bbE\sbr{-\ell'_h\del{y_t \langle w_t,x_t\rangle}\|x_t\|^2\1_{\|x_t\|\ge\tau}}\le\bbE\sbr{\|x_t\|^2\1_{\|x_t\|\ge\tau}}\le d\alpha_1\cdot O\del{d\ln(d/\epsilon)^2}\cdot \frac{\epsilon}{d}=O\del{d\epsilon\ln(d/\epsilon)^2},
  \end{align*}
  where we also use $\ln(1/\epsilon)>1$, since $\epsilon<1/e$.
  To sum up,
  \begin{align*}
    \bbE_{(x_t,y_t)}\sbr[1]{-\ell'_h\del{y_t \langle w_t,x_t\rangle}\|x_t\|^2}\le Cd\del{\cM(w_t)+\epsilon}\cdot\ln(d/\epsilon)^2
  \end{align*}
  for some constant $C$.

  Now taking the expectation with respect to $(x_t,y_t)$ on both sides of
  \cref{eq:sgd_sub_exp_tmp1}, we have
  \begin{align}\label{eq:sgd_sub_exp_tmp}
    \bbE\sbr{\|w_{t+1}-\barr\baru\|^2}\le\|w_t-\barr\baru\|^2-2\eta\del{\cR_h(w_t)-\cR_h(\barr\baru)}+\eta^2Cd\del{\cM(w_t)+\epsilon}\cdot\ln(d/\epsilon)^2.
  \end{align}
  Similarly to the bounded case, we will show that
  \begin{align}\label{eq:sgd_target_sub_exp}
    \bbE\sbr{\min_{0\le t<T}\varphi_t}=O\del{\opt\cdot\ln(1/\opt)+\epsilon},\quad\textup{where}\quad\varphi_t:=\varphi(w_t,\baru).
  \end{align}

  First, \Cref{fact:hinge_ub} implies
  \begin{align*}
    \cR_h(w_t)-\cR_h(\barr\baru) & \ge\cR_h(w_t)-\cR_h(\|w_t\|\baru)-O\del{\del{\opt\cdot\ln(1/\opt)+\epsilon}^2} \\
     & \ge2C_1\|w_t\|\varphi_t^2-C_2\|w_t\|\varphi_t\cdot\opt\cdot\ln(1/\opt)-O\del{\del{\opt\cdot\ln(1/\opt)+\epsilon}^2},
  \end{align*}
  where $C_1:=2R^3/(3U\pi^2)$ and $C_2=(1+2\alpha_1)\alpha_2$.
  Note that if $\varphi_t\le C_2\cdot\opt\cdot\ln(1/\opt)/C_1$, then
  \cref{eq:sgd_target_sub_exp} holds; therefore in the following we assume
  \begin{align}\label{eq:sgd_varphi_lb_sub_exp}
    \varphi_t\ge \frac{C_2}{C_1}\cdot\opt\cdot\ln(1/\opt),
  \end{align}
  which implies
  \begin{align}
    \cR_h(w_t)-\cR_h(\barr\baru) & \ge C_1\|w_t\|\varphi_t^2-O\del{\del{\opt\cdot(1/\opt)+\epsilon}^2} \nonumber \\
     & \ge C_1\varphi_t^2-O\del{\del{\opt\cdot\ln(1/\opt)+\epsilon}^2}, \label{eq:sgd_sub_exp_tmp2}
  \end{align}
  since $\|w\|\ge1$ for all $w\in\cD$.

  On the other hand, for $\opt\le1/e$, \cref{eq:sgd_varphi_lb_sub_exp} and
  \Cref{fact:angle_01}
  imply
  \begin{align*}
    \cM(w_t)=\cR_{0-1}(w_t)\le\opt+2U\varphi_t\le\del{\frac{C_1}{C_2}+2U}\varphi_t.
  \end{align*}
  Let
  \begin{align*}
    C_2:=\frac{C_1}{\del{\frac{C_1}{C_2}+2U}C}.
  \end{align*}
  Note that if $\varphi_t\le\epsilon$, then \cref{eq:sgd_target_sub_exp} is
  true; otherwise we can assume $\epsilon\le\varphi_t$, and let
  $\eta=\frac{C_2\epsilon}{d\ln(d/\epsilon)^2}$, we have
  \begin{align}
    \eta Cd\del{\cM(w_t)+\epsilon}\ln(d/\epsilon)^2 & =\frac{C_2\epsilon}{d\ln(d/\epsilon)^2}Cd\cM(w_t)\cdot\ln(d/\epsilon)^2+\frac{C_2\epsilon}{d\ln(d/\epsilon)^2}Cd\epsilon\cdot\ln(d/\epsilon)^2 \nonumber \\
     & \le C_2\epsilon C\del{\frac{C_1}{C_2}+2U}\varphi_t+C_2C\epsilon^2 \nonumber \\
     & =C_1\epsilon\varphi_t+O\del{\del{\opt\cdot\ln(1/\opt)+\epsilon}^2} \nonumber \\
     & \le C_1\varphi_t^2+O\del{\del{\opt\cdot\ln(1/\opt)+\epsilon}^2}. \label{eq:sgd_sub_exp_tmp3}
  \end{align}

  Now \cref{eq:sgd_sub_exp_tmp,eq:sgd_sub_exp_tmp2,eq:sgd_sub_exp_tmp3} imply
  \begin{align*}
    \bbE\sbr{\|w_{t+1}-\barr\baru\|^2} & \le\|w_t-\barr\baru\|^2-2\eta C_1\varphi_t^2+\eta C_1\varphi_t^2+\eta\cdot O\del{\del{\opt\cdot\ln(1/\opt)+\epsilon}^2} \\
     & =\|w_t-\barr\baru\|^2-\eta C_1\varphi_t^2+\eta\cdot O\del{\del{\opt\cdot\ln(1/\opt)+\epsilon}^2}.
  \end{align*}
  Taking the expectation and average, we have
  \begin{align*}
    \bbE\sbr{\frac{1}{T}\sum_{t<T}^{}\varphi_t^2}\le \frac{\|w_0-\barr\baru\|^2}{\eta C_1T}+\frac{O\del{\del{\opt\cdot\ln(1/\opt)+\epsilon}^2}}{C_1}.
  \end{align*}
  Note that
  \begin{align*}
    \|w_0-\barr\baru\|=\tan(\varphi_0)=O\del[1]{\sqrt{\opt\cdot\ln(1/\opt)+\epsilon}},
  \end{align*}
  and also recall $\eta=\frac{C_2\epsilon}{d\ln(d/\epsilon)^2}$, we have
  \begin{align*}
    \bbE\sbr{\frac{1}{T}\sum_{t<T}^{}\varphi_t^2}\le \frac{O\del{\opt\cdot\ln(1/\opt)+\epsilon}d\ln(d/\epsilon)^2}{C_1C_2\epsilon T}+\frac{O\del{\del{\opt\cdot\ln(1/\opt)+\epsilon}^2}}{C_1}.
  \end{align*}
  Letting $T=\Omega\del{\frac{d\ln(d/\epsilon)^2}{\epsilon^2}}$, we have
  \begin{align*}
    \bbE\sbr{\frac{1}{T}\sum_{t<T}^{}\varphi_t^2} & \le O\del{\opt\cdot\ln(1/\opt)+\epsilon}\cdot\epsilon+O\del{\del{\opt\cdot\ln(1/\opt)+\epsilon}^2} \\
     & =O\del{\del{\opt\cdot\ln(1/\opt)+\epsilon}^2},
  \end{align*}
  and thus \cref{eq:sgd_target_sub_exp} holds.
\end{proof}

\end{document}